\documentclass[oneside,11pt]{article}

\usepackage[utf8]{inputenc}
\usepackage[T1]{fontenc}
\usepackage{scrextend}
\usepackage[english]{babel}

\usepackage{geometry}
\usepackage{graphicx}
\usepackage{hyperref}
\hypersetup{
    hidelinks=true,
}

\usepackage{tikz}

\usepackage{amsthm, amsmath, amssymb}
\newtheorem{theorem}{Theorem}
\newtheorem{proposition}[theorem]{Proposition}
\newtheorem{lemma}[theorem]{Lemma}
\newtheorem{corollary}[theorem]{Corollary}

\theoremstyle{definition}
\newtheorem{definition}[theorem]{Definition}
\newtheorem{example}[theorem]{Example}

\newcommand{\BlackBox}{\rule{1.5ex}{1.5ex}}
\renewenvironment{proof}[1][\proofname]{\par\noindent{\bf #1.~}}{\hfill\BlackBox\\[2mm]}

\usepackage{blindtext}

\usepackage{color}
\usepackage{mathmacros}
\usepackage[sort&compress,capitalize,nameinlink,noabbrev]{cleveref}

\usepackage{natbib}
\def\oh{{\rm o}}
\def\OH{{\rm O}}

\def\indicator#1{{\mathbf 1}\parens{#1}}

\def\Reg{{\rm Reg}}
\def\SliReg{{\rm SliReg}}
\def\RegExp{{\rm RegExp}}
\def\ASSUMPTION#1{({\bf A#1})}
\def\kl{{\rm kl}}

    \def\STEP#1{(\textbf{STEP~#1})}
\def\ASSUMPTION#1{(\textbf{A#1})}

\usepackage{lastpage}

\begin{document}

\title{\bf The Sliding Regret in Stochastic Bandits:
    \\
    \bf Discriminating Index and Randomized Policies
}

\author{%
    \bf Victor Boone\\
    \small \texttt{victor.boone@univ-grenoble-alpes.fr}\\
    \small Univ.~Grenoble Alpes, Inria, CNRS, Grenoble INP, LIG, 38000 Grenoble, France
}
\date{}


\maketitle

\begin{abstract}
    This paper studies the one-shot behavior of no-regret algorithms for stochastic bandits.
    Although many algorithms are known to be asymptotically optimal with respect to the expected regret, over a single run, their \emph{pseudo-regret} seems to follow one of two tendencies: it is either smooth or bumpy. 
    To measure this tendency, we introduce a new notion: the \emph{sliding regret}, that measures the worst pseudo-regret over a time-window of fixed length sliding to infinity. 
    We show that randomized methods (e.g.~Thompson Sampling and MED) have optimal sliding regret, while index policies, although possibly asymptotically optimal for the expected regret, have the worst possible sliding regret under regularity conditions on their index (e.g.~UCB, UCB-V, KL-UCB, MOSS, IMED etc.). 
    We further analyze the average bumpiness of the pseudo-regret of index policies via the \emph{regret of exploration}, that we show to be suboptimal as well. 
\end{abstract}

\section{Introduction}

In the stochastic multi-armed bandit problem, an agent picks actions $A_1, A_2, \ldots,$ sequentially and receives rewards accordingly, where each reward $R_t$ is generated by an underlying fixed (but unknown) probability distribution ${\rm F}(A_t)$ associated to the picked action $A_t \in \set{1 ,\ldots, K}$. 
Her goal is to maximize her expected aggregate reward over time. 
This task is equivalent to minimizing the \emph{regret}, given by the expected cumulated reward obtained by only pulling the optimal arm against the actually achieved rewards $T\mu^* - \sum_{t=1}^T R_t$, where $\mu_a$ is the expectation of ${\rm F}_a$ and $\mu^* := \max_a \mu_a$ is maximal achievable expected reward. 
An action (or arm) $a$ such that $\mu_a = \mu^*$ will be said optimal, and suboptimal otherwise. 
A related and key quantity is the \emph{pseudo-regret} which is the partial expectation taken with respect to the rewards, that is, 
\begin{equation}
    \label{equation:regret}
    T\mu^* - \sum_{t=1}^T \mu_{A_t}.
\end{equation}
This regret minimization problem is now very well understood.
Lower bounds of achievable expected regret are known (see \citealt{lai_asymptotically_1985,auer_gambling_1995}) and achieved by multiple methods, for example Thompson Sampling (\citealt{kaufmann_thompson_2012}), MED (\citealt{honda_asymptotically_2010}), IMED (\citealt{honda_non-asymptotic_2015}), KL-UCB (\citealt{garivier2011kl,maillard2011finite}) or MOSS (\citealt{audibert2009minimax}). 

When applied to real-world tasks, what usually matters is the performance over a single run. 
Yet, although many of the previously mentioned methods are asymptotically optimal in expectation, their trajectory behaviors differ significatively. 
This is illustrated in \cref{figure:typical}, plotting the typical pseudo-regrets of two popular algorithms: UCB by \cite{auer_using_2002} and Thompson Sampling (TS) by \cite{thompson_likelihood_1933}.

\begin{figure}
    \centering
    \includegraphics[width=.49\linewidth]{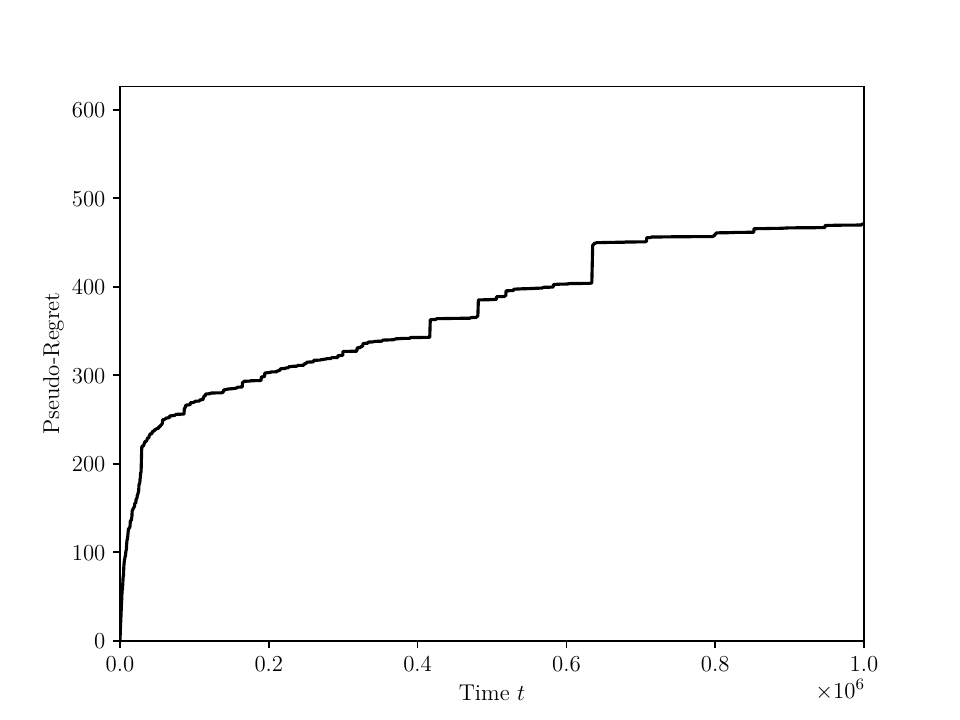}
    \hfill
    \includegraphics[width=.49\linewidth]{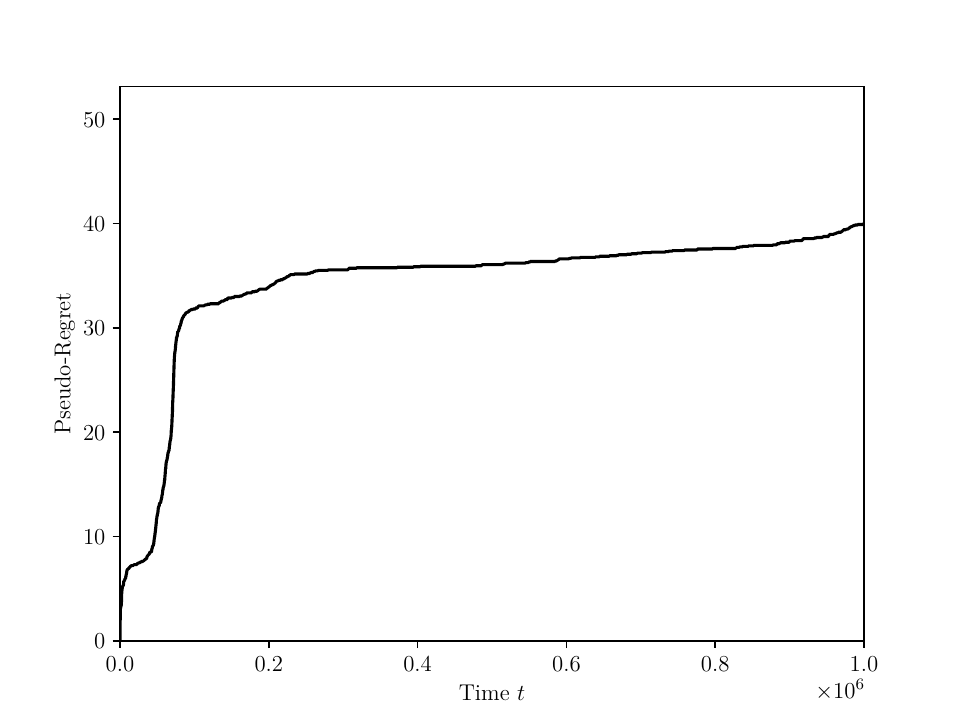}
    \caption{
        \label{figure:typical}
        Typical one-shot pseudo-regret of UCB (left) and Thompson Sampling (right).
        The model is a two-armed bandit with Bernoulli rewards ${\rm B}(0.85), {\rm B}(0.8)$. 
    }
\end{figure}

The difference is striking. 
UCB has bumpy pseudo-regret and alternates between periods of time when it pulls the best and a suboptimal arm, meaning that it repeatedly pulls a bad arm several times in a row.
In opposition, the pseudo-regret of Thompson Sampling is smooth and the algorithm seems to pull suboptimal arms sporadically over time. 

These two trajectory portraits are in fact the two representatives of most existing algorithms for stochastic bandits. 
UCB showcases the typical one-shot pseudo-regret of index policies while TS illustrates the ones of randomized policies. 

\vspace{-1em}
\subsection{Our Contributions}

In this paper, we explain the phenomenon reported in \cref{figure:typical}. 
To simplify the discussion, the results are established for two-arm Bernoulli bandits.
To measure the asymptotic bumpiness of the pseudo-regret, we introduce a new learning metric that we call the \emph{sliding regret}, given by the worst pseudo-regret on time-windows of fixed length sliding to infinity.
Our first result, \cref{theorem:agnosticity}, provides a general condition to guarantee that a given policy have small sliding regret, later used to show that Thompson Sampling and MED have optimal sliding regret. 
Our second result, \cref{theorem:general indexes}, states that all index policies have linear sliding regret provided that the index meets some regularity conditions. 
An \emph{index policy} (see \citealt[35.4]{lattimore2020bandit}) is an algorithm that, out of its current observations, associates a real-valued index to each arm then picks the arm with maximal index. 
Our result covers all classical index policies in the literature, such as UCB (\citealt{auer_using_2002}), UCB-V (\citealt{audibert2007tuning}), MOSS (\citealt{audibert2009minimax}), KL-UCB (\citealt{garivier2011kl,maillard2011finite}), IMED (\citealt{honda_non-asymptotic_2015}) as well as their variants. 

The study of the sliding regret of index policies indicates that such algorithms have a tendency to pick suboptimal arms several times in a row at exploration episodes. 
An \emph{exploration episode} (see \citealt{boone_regret_2023}) is a time-instant $t$ such that $A_{t-1}$ is optimal but $A_t$ is not. 
What happens at these critical time-instants is what makes the sliding regret of index policies linear, as the probability to pick a suboptimal arm $T$ times in a row starting from an exploration episode $t$ is positive and does not vanish with $t$. 
We go beyond this result by asking the following: 
What is the expected regret starting from an exploration episode?
This leads to the \emph{regret of exploration}, which is shown to be optimal for Thompson Sampling and MED but sub-optimal for classical index policies. 

Why does this all matter?
When considering the historical motivation of the stochastic bandit problem (\citealt{thompson_likelihood_1933}) where pulling arm $a$ is providing drug $a$ to a patient, and $R_t$ is whether the $t$-th patient is cured or not, the bumps observed in the pseudo-regret are several patients in a row being provided the wrong medicine.
For this application, providing the drug which is known to be empirically worse than the other several times in succession is unfair to patients. 
This can even be exploited by them; If you know that the bad drug B has been provided recently, wouldn't you wait a little bit for the algorithm to stop providing drug B, so that you can get drug A for sure instead?
In such a scenario, small sliding regret guarantees appear to be important. 

\section{Preliminaries}

In this paper, we focus on stochastic bandits with two arms of Bernoulli distributions with respective means $\mu_1, \mu_2 \in (0, 1)$.
These two assumptions are made solely to simplify the discussion, as our proofs and methods are not specific to two-arm bandits nor Bernoulli distributions.
The distribution of arm $a$ is denoted ${\rm F}_a$ and we denote $\Pr_{\rm F}(-)$ and $\E_{\rm F}[-]$ the associated probability and expectation operators. 
Whenever the distribution is clear in the context, the subscript ${\rm F}$ is dropped. 
We further assume, up to a permutation of arms, that $\mu_2 < \mu_1$. 
Thus arm $1$ is optimal and arm $2$ is suboptimal.
A \emph{policy} (or algorithm) is a (possibly randomized) decision rule that, to each history of observations $(A_1, R_1, \ldots, A_{t-1}, R_{t-1})$, associates the (possibly random) next arm $A_t$ to be pulled; After pulling the arm $A_t$, the policy observes $R_t \sim {\rm B}(A_t)$, generated independently of the history and the inner randomness of the policy. 
The regret of a policy is given $T \mu_1 - \sum_{t=1}^T R_t$.
Its partial expectation with respect to the rewards is the \emph{pseudo-regret}, visually:
\begin{equation}
    {\Reg}(T) := T \mu_1 - \sum_{t=1}^T \mu_{A_t}.
\end{equation}
The indicator function is denoted $\indicator{-}$. 
During a run of a policy and for every arm $a = 1,2$, we keep track of the number of visits with $N_a(t) := \sum_{i=1}^{t-1} \indicator{A_i = a}$.
We will have $S_{a, n}$ and $\hat \mu_{a, n}$ denote the number of successes (when the reward equals one) and empirical mean of arm $a$ after $n$ draws of it.
$S_a(t) := S_{a, N_a(t)}$ and $\hat \mu_a(t) := \hat \mu_{a, N_a(t)}$ will denote the associated number of successes and empirical mean at time $t$. 

%
%
\section{The Sliding Regret and the Behavioral Robustness to Local Histories}

The presence of bumps observed in \cref{figure:typical} is related to the slope of the pseudo-regret, which is given by the pseudo-regret difference between two points in time.
This consists in its truncation to a given time-window. 
Accordingly, to study the local behavior of the pseudo-regret, we study its truncation to time-windows of fixed length sliding to infinity. 

\begin{definition}
    The asymptotic sliding regret (or \emph{sliding regret} for short) is given by
    $$
        \SliReg(T) := \limsup_{t \to \infty} \parens{T\mu^* - \sum_{i=1}^T \mu_{A_{t+i}} }.
    $$
\end{definition}
The sliding regret is a non-negative quantity that measures the presence and the amplitude of the local changes of the pseudo-regret in the asymptotic regime. 
It is a new learning metric that lives independently from the regret as, as we will see, no-regret algorithms in the literature present two tendencies: those with small sliding regret and high sliding regret, embodied by Thompson Sampling and UCB respectively. 

As a starting point of our analysis, observe that an algorithm that visits all arms infinitely often must have sliding regret at most $\mu_1 - \mu_2$, as stated by the proposition below. 
\begin{proposition}
    Consider an algorithm that, for all distribution on arms, have sublinear expected regret. 
    Then it has sliding regret bounded as:
    \begin{equation}
        \notag
        \mu_1 - \mu_2 \le \SliReg(T) \le T(\mu_1 - \mu_2). 
    \end{equation}
\end{proposition}
This is a direct consequence of Proposition~\ref{proposition:consistent policies}, established in \cref{section:consistent algorithms}. 

There is a world between the lower and the upper bound and yet, the lower bound is usually achieved with randomized methods (such as TS) while the upper bound is reached with index policies (such as UCB). 
But associating small sliding regret guarantees with randomization is slightly misleading; this is rather a question of how the policy behaves depending on its recent history. 

\subsection{Behavioral Robustness to Local Histories}

The main result of this section is \cref{theorem:agnosticity}. 
This theorem states that if, regardless of the recent history (e.g., a bad arm has been pulled), the probability of picking a suboptimal action remains small, then the sliding regret is small.
This is the property that we refer to as the behavioral robustness to local histories. 
It is precisely the property that UCB does not satisfy and that will lead to suboptimal sliding regret later on. 

%

\begin{theorem}
    \label{theorem:agnosticity}
    Let $\pi$ a policy such that there exists a sequence of events $(E_t : t \ge 1)$ with $\Pr(\exists t, \forall s \ge T: E_s) = 1$, that satisfies:
    \begin{equation}
        \label{equation:agnosticity}
        \exists d > 0, \forall i \ge 1:
        \quad
        \Pr (A_{t+i} \ne 1 |H_{t:t+i}, E_t) = \OH\parens{\frac 1{t^d}}
    \end{equation}
    where $H_{t:t+i}$ is the truncated history $(A_{t+j}, R_{t+j})_{0 \le j < i}$. 
    Then $\SliReg(\pi; T) \le \lfloor\frac 1d\rfloor(\mu_1-\mu_2)$. 
\end{theorem}
\begin{proof}
    Let $n > \frac 1d$ an integer. 
    We show that $\Pr(\forall t, \exists s \ge t: \Reg(s;s+T) \ge n(\mu_1 - \mu_2)) = 0$. 
    Remark that if $\Reg(s; s+T) \ge n(\mu_1 - \mu_2)$, there exists a set $I \subseteq \set{0, \ldots, T-1}$ of size $n$ such that, for all $i \in I$, $A_{s+i} = 2$. 
    Denote $\Lambda_n$ the collection of subsets of $\set{0, \ldots, T-1}$ of size $n$ and fix $I \in \Lambda_n$ whose elements are denoted $i_1 < i_2 < \ldots < i_n$. 
    We have:
    \begin{align*}
        & \Pr (\forall i \in I, A_{t+i} = 2; E_t)
        \\
        & = 
        \Pr(A_{t+i_n} = 2 \mid E_t, (\forall i \in I\setminus\set{i_n}, A_{t+i} = 2))
        \Pr(\forall i \in I\setminus\set{i_n}, A_{t+i} = 2; E_t)
        \\
        & =
        \OH\parens{\frac 1{t^d}}
        \cdot \Pr(\forall i \in I\setminus\set{i_n}, A_{t+i} = 2; E_t)
        \\
        & = \ldots
        \\
        & = \OH\parens{\frac 1{t^{nd}}} 
        \Pr(E_t) = \OH\parens{\frac 1{t^{nd}}} .
    \end{align*}
    Because $E_s$ satisfies $\Pr(\liminf E_s) = 1$, check that for all sequence of events $(F_s)$, $\Pr(\forall t, \exists s \ge t: F_s) = \Pr(\forall t, \exists s \ge t: E_s, F_s)$.
    We complete the proof with:
    \begin{align*}
        \Pr(\forall t, \exists s \ge t: \Reg(s; s+T) \ge n(\mu_1-\mu_2))
        & = 
        \lim_{t \to \infty} \Pr \parens{
            \exists s \ge t:
            \Reg(s; s+T) \ge n(\mu_1-\mu_2)
        }
        \\
        & =        
        \lim_{t \to \infty} \Pr \parens{
            \exists s \ge t:
            \Reg(s; s+T) \ge n(\mu_1-\mu_2), E_s
        }
        \\
        & \le 
        \lim_{t \to \infty} 
        \sum\nolimits_{s \ge t}
        \Pr \parens{
            \Reg(s; s+T) \ge n(\mu_1-\mu_2), E_s
        }
        \\
        & \le 
        \lim_{t \to \infty} 
        \sum\nolimits_{s \ge t}
        \Pr \parens{
            \exists I \in \Lambda_n, \forall i \in I : A_{s+i}=2; 
            E_s
        }
        \\
        & \le 
        \lim_{t \to \infty} 
        \sum\nolimits_{s \ge t}
        T^n \OH\parens{ \frac1{s^{nd}} }
        = 0
    \end{align*}
    because $nd > 1$. 
    So $\SliReg(T) \le (n-1)(\mu_1 - \mu_2)$. 
\end{proof}

In order to apply the theorem, one has to find the right sequence of events $E_t$ such that \eqref{equation:agnosticity} is satisfied. 
This event usually characterizes what we will refer to as the \emph{asymptotic regime} of an algorithm, consisting in concentration guarantees for the empirical data of the algorithm as well as convergence of the visit rate of the suboptimal arm. 
A complete example is provided with Thompson Sampling. 

\subsection{Application: Thompson Sampling}

Thompson sampling (TS) is a Bayesian policy that, at time $t$, sample estimates of the arms' values from its posterior distribution and picks the arm with highest estimate. 
In the chosen Bernoulli setting, when the initial prior of TS is a tensor product of uniform distributions over $[0, 1]$, the posteriors are Beta distributions and TS's estimates are sampled as:
\begin{equation}
    \theta_a(t) \sim {\rm Beta}(1+S_a(t), 1+N_a(t)-S_a(t)).
\end{equation}
The expected regret of TS is pretty well understood.
In the frequentist formulation of the multi-armed bandit problem, the regret is $\OH(\log(T))$ (\citealt{agrawal_analysis_2012}) and the multiplicative coefficient is the best possible, making TS an asymptotically optimal algorithm, see \cite{kaufmann_thompson_2012,pmlr-v31-agrawal13a}. 
We additionally show that its sliding regret is optimal. 

\begin{theorem}
    \label{theorem:ts sliding regret}
    Thompson Sampling has optimal sliding regret $\SliReg({\rm TS}; T) = \mu_1 - \mu_2$. 
\end{theorem}

The complete proof is pretty tedious and deferred to the appendix. 
Also we sketch the main lines of it. 
    The goal is to invoke \cref{theorem:agnosticity}, and this is achieved by characterizing the asymptotic behavior of TS, consisting in estimates of the sampling rates $\Pr(A_t=a)$ of the algorithm, estimates of the visit rates $N_a(t)$ as well as convergence of its empirical data. 

Because the expected regret is sublinear, all arms are visited infinitely often, hence posteriors concentrate around the true means $\mu_1, \mu_2$, meaning that $\hat \mu_a(t)$ eventually converges to $\mu_a$ for $a = 1,2$.
A second known property of the asymptotic regime is that for some $b > 0$, $\sum \Pr(N_1(t) \le t^b) < \infty$, see \cite[Proposition 1]{kaufmann_thompson_2012}.
So by Borel-Cantelli's lemma, $\Pr(\liminf(N_1(t) > t^b))=1$. 
Together with a combination of the Beta-Bernoulli trick (\cite{agrawal_analysis_2012}) and Sanov' Theorem, the sampling rates of Thompson Sampling are bounded as follows:
For all $\epsilon>0$, there exists a sequence of events $(F_t^\epsilon)$ with $\Pr(\liminf F_t^\epsilon) = 1$ such that 
\begin{equation}
    \label{equation:ts sample rates}
    e^{-(1+c(\epsilon))N_2(t)\kl(\mu_2, \mu_1)}
    \le 
    \Pr (A_t = 2 \mid F_t^\epsilon)
    \le 
    e^{-(1-c(\epsilon))N_2(t)\kl(\mu_2, \mu_1)},
\end{equation}
    where $c(\epsilon)$ is a $o(1)$ when $\epsilon$ vanishes. 
    We use \eqref{equation:ts sample rates} to show that $N_2(t) \sim \log(t)/\kl(\mu_2,\mu_1)$.
More precisely, we show that for all $\epsilon > 0$, the event
$$
    E_t^\epsilon := 
    \parens{
        \forall a, \abs{\hat \mu_a(t) - \mu_a} < \epsilon
    }
    \cap 
    \parens{
        \abs{N_2(t) - \tfrac{\log(t)}{\kl(\mu_2,\mu_1)}} < \epsilon \cdot \tfrac{\log(t)}{\kl(\mu_2,\mu_1)}
    }
$$
    holds eventually (the limit inferior is almost-sure). 
    On this event and relying on \eqref{equation:ts sample rates}, we establish that
$$
    \Pr(A_t=2\mid E^\epsilon_t, H_{t:t+i}) = \OH\parens{t^{-1+\underset{\epsilon \to 0}{\oh}(1)}}.
$$
    Applying \cref{theorem:agnosticity}, we obtain $\SliReg(T) \le \lfloor \frac 1{1+\oh(1)}\rfloor(\mu_1 - \mu_2)$.
    Making $\epsilon$ go to zero, we obtain optimal sliding regret guarantees for TS. 


\subsection{Application: MED}


MED (\cite{honda_asymptotically_2010}) is a randomized algorithm that, at time $t$, samples the arm $a$ with probability proportional to $\exp(-N_a(t)\kl(\hat \mu_a(t), \hat \mu^*(t)))$. 
MED is known to have asymptotically optimal expected regret (refer to the original paper).
As the empirical estimates converge, the sampling rate of the arm $a = 2$ is approximately $\exp(-N_2(t) \kl(\mu_2, \mu_1))$, which is essentially the same as Thompson Sampling's in the asymptotic regime. 
Therefore, the analysis of its sliding regret is similar to Thompson Sampling's. 

\begin{theorem}
    MED has optimal sliding regret
    $\SliReg({\rm MED}; T) = \mu_1 - \mu_2$.
\end{theorem}

%


\section{The Bumpy Regret of UCB}

In this section, we show that unlike TS and MED, UCB does not have good sliding regret guarantees. 
In fact, the sliding regret of UCB is the worst possible as shown with \cref{theorem:ucb sliding}.

The UCB algorithm from \cite{auer_using_2002} is an index algorithm rooted in the \emph{optimism-in-face-of-uncertainty} principle, that can be traced back at least to \cite{lai_asymptotically_1985}. 
At time $t$, it picks the arm maximizing the index
\begin{equation}
    \label{equation:ucb index}
    \hat \mu_a(t) + \sqrt{\frac{2\log(t)}{N_a(t)}}
\end{equation}
which is $+\infty$ if $N_a(t) = 0$ by convention. 
Expected regret guarantees in $\OH(\log(T))$ can be found in the original paper.
This algorithm is arguably the basis of many algorithms and has been thoroughly investigated. 
This is why, to build intuition on how index algorithms typically behave, we dedicate this section to the analysis of the almost sure regime of UCB. 

Thankfully, the almost-sure behavior of UCB at infinity is well-behaved.
Eventually $\hat \mu_a(t)$ converges to $\mu_a$ and the visit rates of arms are such that the index of both arms \eqref{equation:ucb index} are approximately equal. 
In fact, $N_1(t) \sim t$ and $N_2(t) \sim \frac 2{(\mu_1-\mu_2)^2}\log(t)$ when time goes to infinity, see Proposition~\ref{proposition:asymptotic ucb}.

\begin{proposition}
    \label{proposition:asymptotic ucb}
    For all $\epsilon > 0$ and when running UCB, both of the following hold:
    \begin{itemize}
        \item[(1)] 
            $\Pr(\exists t, \forall s \ge t: \forall a, \abs{\mu_a(s) - \mu_a(s)} < \epsilon) = 1$;

        \item[(2)]
            $\Pr\parens{
                \exists t, \forall s \ge t:
                \abs{N_2(t) - 2\parens{\tfrac 1{\mu_1-\mu_2}}^2\log(t)}
                < \epsilon \cdot 2\parens{\tfrac 1{\mu_1-\mu_2}}^2\log(t)
            } = 1.$
    \end{itemize}
\end{proposition}

The proof of Proposition~\ref{proposition:asymptotic ucb} is provided in \cref{section:ucb analysis}. 

\subsection{The Sliding Regret of UCB}

The analysis is driven by the behavior observed in \cref{figure:typical}. 
UCB pulls every arm infinitely often, and every time it does pick the suboptimal arm, the probability that it picks it again in the next round is high. 
Intuitively speaking, this happens because when UCB picks the suboptimal arm $a=2$ and receives full reward $R_t = 1$, the empirical estimate $\hat \mu_2(t)$ increases enough so that UCB ``thinks'' that it has been sub-sampled. 
In other words, in the asymptotic regime of UCB, if a suboptimal arm provides promising rewards, UCB will keep pulling it to ``make sure'' that this arm's estimate is not misestimated. 
This means that the central condition of \cref{theorem:agnosticity} is not met by UCB. 
The time instants when UCB starts pulling suboptimal arms are called \emph{exploration episodes}, and are formally given by the increasing sequence of stopping times:
\begin{equation}
    \label{equation:exploration episodes}
    \tau_1 := \inf\set{t : A_t=2},
    \quad
    \tau_{k+1} := \inf\set{t > \tau_k: A_t = 2 \wedge A_{t-1} = 1}.
\end{equation}
Since all arms are pulled infinitely often, all these are almost surely finite.

\begin{lemma}
    \label{lemma:ucb non agnosticity}
    Consider running UCB, and fix $T > 0$. 
    There exists a sequence of events indexed by exploration episodes $(E_{\tau_k})$ with $\Pr(\liminf_k E_{\tau_k}) = 1$, such that, for all sequence $(U_t: t\ge1)$ of $\sigma(H_t)$-measurable events:
$$
    \Pr \parens{\forall i < T: A_{\tau_k+i} = 2 \mid E_{\tau_k}, U_{\tau_k}}
    \ge 
    \mu_2^T.
$$
\end{lemma}

The additional sequence $(U_t)$ informs that the above lower bound is resilient to pollution of the history. 
The event $E_{\tau_k}$ is mostly about the concentrations of empirical means $\hat \mu_a(t)$ and of visit rates, given by Proposition~\ref{proposition:asymptotic ucb}.
The main line of the proof is to estimate the evolution of UCB's index \eqref{equation:ucb index} with respect to $\hat \mu_a(t)$, $t$ and $N_a(t)$ for $a = 1, 2$, and to show that UCB keeps picking the suboptimal arm while it provides optimal reward. 
It follows that the sliding regret is linear.

\begin{theorem}
    \label{theorem:ucb sliding}
    UCB has the worst possible sliding regret $\SliReg(T) = T(\mu_1-\mu_2)$. 
\end{theorem}

\begin{proof}
    Since $\tau_{k+T} > \tau_{k} + 2T$, the events $(\exists i < T: A_{\tau_k+i}\ne 2)$ and $(\exists i < T: A_{\tau_{k+T}} \ne 2)$ do not overlap. 
    Denote $F_{\tau_{\ell T}} := (\exists i < T: A_{\tau_{\ell T}+i} \ne 2)$ and let $E_{\tau_k}$ the event given by Lemma~\ref{lemma:ucb non agnosticity} for a fixed $\epsilon < \mu_2$. 
    Observe that $\Pr(\forall \ell \ge k: E_{\tau_{\ell T}} \cap F_{\tau_{\ell T}})$ can be put in the form: 
    $$
        \prod_{\ell \ge k} 
        \Pr\parens{
            \exists i < T: A_{\tau_{\ell T} +i} \ne 2 
            \mid 
            E_{\tau_{\ell T}}, 
            U_{\tau_{(\ell-1)T}+T}
        }
        \Pr \parens{
            E_{\tau_{\ell T}}
            \mid
            U_{\tau_{(\ell-1)T}+T}
        }
    $$
    where $U_{\tau_{(\ell-1)T}+T} := \bigcap_{m \le \ell-1} (E_{\tau_{mT}} \cap F_{\tau_{mT}})$ is a $\sigma(H_{\tau_{(\ell-1)T}+T})$-measurable event. 
    Applying Lemma~\ref{lemma:ucb non agnosticity}, we obtain:
    $$
        \Pr(\forall \ell \ge k: E_{\tau_{\ell T}} \cap F_{\tau_{\ell T}})
        \le 
        \prod\nolimits_{\ell \ge k} \parens{1 - (\mu_2-\epsilon)^T} = 0.
    $$
    It follows that:
    $$
        \Pr(\forall l \ge k: F_{\tau_{\ell T}})
        \le
        \Pr\parens{\forall l \ge k: (E_{\tau_{\ell T}})^\complement}
        +
        \Pr(\forall \ell \ge k: E_{\tau_{\ell T}} \cap F_{\tau_{\ell T}})
        = 
        \Pr\parens{\forall l \ge k: (E_{\tau_{\ell T}})^\complement}.
    $$
    But since $\Pr(\liminf_k E_{\tau_{k T}}) = 1$, the above RHS goes to zero as $k \to \infty$. 
    Therefore, we obtain $\Pr(\forall k, \exists \ell \ge k: \forall i < T, A_{\tau_\ell+i} = 2) = 1$, proving $\SliReg({\rm UCB}; T) = T(\mu_1-\mu_2)$. 
\end{proof}

With the same proof techniques than Lemma~\ref{lemma:ucb non agnosticity}, the above result can be further refined.
When UCB receives full reward from the suboptimal arm, the associated index increases significatively so that, not only UCB will pick the suboptimal arm again in the next round, but it will also pick it in the next round, independently of the observed reward. 
Roughly speaking, if UCB receives many promising rewards in a row for the suboptimal arm, the associated index is polluted and UCB will blindly pick it again many times in succession, independently of the feedback.

\begin{proposition}
    \label{proposition:ucb monkey}
    Fix $T > 0$ and assume that we are running UCB.
    There exists an increasing sequence of almost-surely finite stopping times $(\sigma_k : k \ge 1)$ s.t.,
    $$
        \Pr(\Reg(\sigma_k; \sigma_k+T) \ge (\mu_1-\mu_2)T) = 1.
    $$
\end{proposition}
For the construction of $(\sigma_k)$, refer to \cref{section:ucb fails}.

Regarding Lemma~\ref{lemma:ucb non agnosticity}, the lower bound of for $\Pr(\forall i < T: A_{\tau_k+i} = 2)$ is decreasing exponentially fast with $T$. 
Even though \cref{theorem:ucb sliding} states that the pseudo-regret of UCB makes arbitrarily large jumps infinitely often, how rare are these large jumps? 
While it is known that the infinite monkey eventually writes the complete works of William Shakespeare, the expected time that the animal requires to eventually write the first sentence of \emph{Romeo and Juliet} is stupidly large. 

\subsection{The Regret of Exploration}

If UCB has a tendency to pick the suboptimal arm many times at exploration episodes $(\tau_k: k \ge 1)$, see \eqref{equation:exploration episodes}, how significant is this tendency?
We now investigate the expected regret starting from $\tau_k$ using the regret of exploration, first introduced in the work of \cite{boone_regret_2023}.

\begin{definition}
    The \emph{regret of exploration} of an algorithm is the quantity:
    \begin{equation}
        \RegExp(T) := \limsup_{k \to \infty} \E[\Reg(\tau_k; \tau_k+T)].
    \end{equation}
\end{definition}

As soon as an algorithm visits every arm infinitely often, the regret of exploration is well-defined, although the notion of \emph{exploration episode} is less natural for randomized algorithms such as TS or MED than it is for index algorithms like UCB. 
The regret of exploration is an alternative measure to the sliding regret, also quantifying the tendency of an algorithm to aggregate suboptimal play. 
The two are linked as follows.

\begin{proposition}
    For every algorithm, $\RegExp(T) \le \E[\SliReg(T)]$. 
\end{proposition}
\begin{proof}
    Since $\tau_k < \tau_{k+1}$, we have $\tau_k \ge k$. 
    Therefore:
    \begin{align*}
        \RegExp(T) 
        := \inf_k \sup_{\ell \ge k} \E[\Reg(\tau_\ell;\tau_\ell+T)]
        &\le \inf_t \sup_{s \ge t} \E[\Reg(s; s+T)]
        \\
        &\le \inf_t \E\brackets{\sup_{s \ge t} \Reg(s; s+T)}.
    \end{align*}
    By definition, $\Reg(s; s+T) \in [0, T(\mu_1-\mu_2)]$ almost surely, so is bounded.
    By the Bounded Convergence Theorem, $\inf_t \E\brackets{\sup_{s \ge t} \Reg(s; s+T)}=\E[\inf_t \sup_{s \ge t}\Reg(s; s+T)]$.
    We readily obtain: $\RegExp(T) \le \E[\SliReg(T)]$.
\end{proof}

Combined with \cref{theorem:ts sliding regret}, this shows that Thompson Sampling has optimal regret of exploration. 
The same goes for MED, since MED also has sliding regret $\mu_1 - \mu_2$. 

\begin{corollary}
    Thompson Sampling and MED have optimal regret of exploration, that is, for $\pi = {\rm TS}$ or ${\rm MED}$, $\RegExp(\pi; T) = \mu_1 - \mu_2$. 
\end{corollary}

The regret of exploration of UCB is shown to be lower bounded by $C(T)(\mu_1-\mu_2)$ where $C(T)$ is bounded away from $1$.
We show that at exploration episodes and in the asymptotic regime, UCB behaves like a random walk with a negative drift, and that the regret exploration is related to the reaching time to $\R_-$. 
The proof is found in \cref{section:ucb exploration}.

\begin{theorem}
    \label{theorem:ucb regret of exploration}
    Let $(X_t:t\ge 1)$ a sequence of i.i.d.~random variables with distribution ${\rm B}(\mu_2)$. 
    Let $\sigma_T$ the stopping time
    $
        T \wedge
        \inf\set{ 
            t \ge 1: 
            -\frac{\mu_1 - \mu_2}2 + \frac 1t \sum_{i=1}^t (X_t - \mu_2) \le 0
        }.
    $
    For all $T\ge 1$, we have $\RegExp({\rm UCB}; T) \ge (\mu_1-\mu_2)\E[\sigma_T]$. 
\end{theorem}

How tight is this result? 
How close is $\E[\Reg(\tau_k;\tau_k+T)]$ to $(\mu_1-\mu_2)\E[\sigma_T]$ in practice?
To proceed, we estimate as a function of $t$ the expected regret at exploration episodes near $t$, consisting in $\RegExp'(t;T) := \E[\Reg(t; t+T)| \exists k, t = \tau_k]$. 
To estimate this function, we repeatedly run the algorithm to obtain a family $S$ of samples of $(\tau_k, \Reg(\tau_k;\tau_k+T))$.
Then, we approximate $\RegExp'(t; T)$ as the averages of $y$ for $(x,y) \in S$ such that $\abs{x - t} < W$ where $W$ is a parameter of the approximation. 
As shown in \cref{figure:regret of exploration}, we indeed confirm that in practice, the expected regret during an exploration episode seems to converge to the anticipated lower bound $(\mu_1- \mu_2) \E[\sigma_T]$ reasonably quickly. 

\begin{figure}
    \centering
    \includegraphics[width=.7\linewidth]{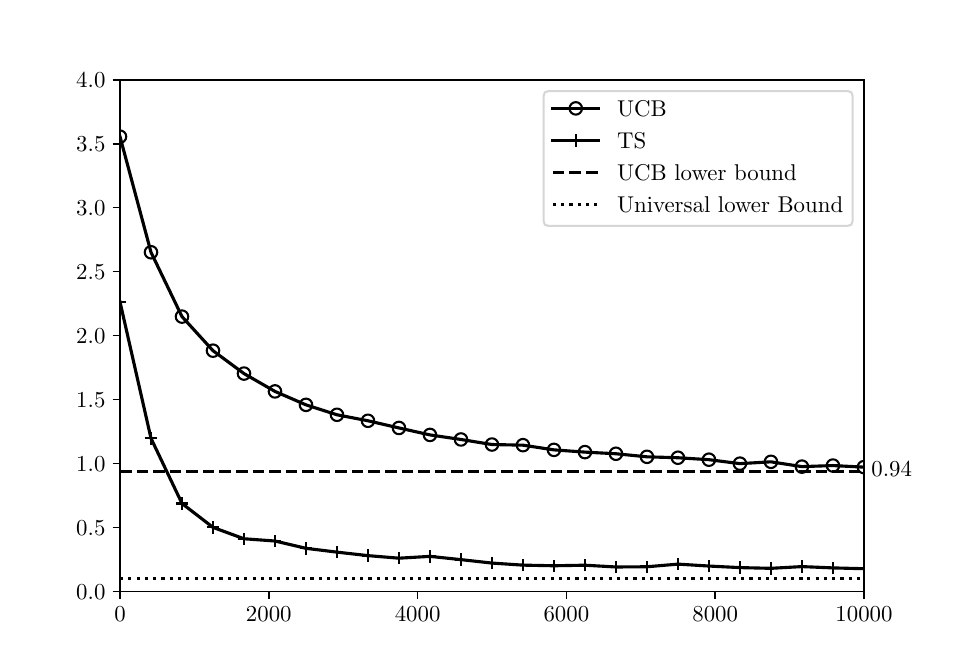}
    \caption{
        \label{figure:regret of exploration}
        The approximate of $t \mapsto \RegExp'(t; 100)$ for UCB and TS executed on the two-arm bandit $({\rm B}(0.9), {\rm B}(0.8))$, averaged over 10k runs, with $W = 200$ (see below).
    }
\end{figure}

\section{General Index Algorithms}

The behavior reported in \cref{figure:typical} and analyzed in the previous section is not specific to UCB.
In this section, we generalize the analysis of UCB to most index policies of the literature.
We provide a set of conditions under which an index policy has linear sliding regret, see \cref{theorem:general indexes}.

An \emph{index policy} is an algorithm that, out of past observations, associates to every arm a numerical value called the \emph{index} of the arm, and pulls the arm with maximal index. 
In the sequel, we consider indexes of the form
\begin{equation}
    \label{equation:index}
    I(\hat \mu_a(t), \hat \mu_{-a}(t), N_a(t), t) \in [0, I_{\max}]
\end{equation}
where $I_{\max} \in (0, +\infty]$ is the maximal value that the index can reach (possibly infinite), $\hat \mu_a(t)$ the empirical value of the considered arm, $\hat \mu_{-a}(t)$ the collection of the empirical values of other arms, $N_a(t)$ the current number of visits of the arm and $t$ the time. 
Accordingly, at time $t$, the algorithm picks $A_t \in \argmax_{a} I(\hat \mu_a(t), \hat \mu_{-a}(t), N_a(t), t)$. 
Remark that the ordering of $\hat \mu_a(t)$ and $\hat \mu_{-a}(t)$ is important because $I(\hat \mu_1(t), \hat \mu_2(t), N_1(t), t)$ refers to the index of arm $a=1$ while $I(\hat \mu_2(t), \hat \mu_1(t), N_2(t), t)$ refers to the index of arm $a=2$; we will write $I_a(t)$ and $I_a(\hat \mu(t), N_a(t), t)$ for simplicity. 

The goal of this section is to generalize \cref{theorem:ucb sliding} and \cref{theorem:ucb regret of exploration} to general index policies. 
Our final result is summarized with \cref{theorem:general indexes}. 
Of course, it is impossible to grasp all index policies within a single result, so the index has to meet regularity conditions for our result to be applicable. 
We design a set of nine conditions \ASSUMPTION{1-9}.
All of them are met by classical existing indexes. 

    \begin{table}
    \begin{center}
    \begin{tabular}{c|c|c|c}
        Algorithm & Original Index & Reworked index & $I_{\rm max}$
        \\ \hline
        UCB & $\hat \mu_a(t) + \sqrt{\frac{2\log(t)}{N_a(t)}}$ & - & $\infty$
        \\
        MOSS & $\hat \mu_a(t) + \sqrt{\frac{\log(t/KN_a(t))}{N_a(t)}}$ & - & $\infty$
        \\
        KL-UCB & $\max\set{\mu : N_a(t)\kl(\hat \mu_a(t), \mu) \le \log(t)}$ & - & $1$
        \\
        IMED & $ - N_a(t) \kl(\hat \mu_a(t), \hat \mu^*(t)) - \log N_a(t)$ & $\frac{\log(t)}{N_a(t) \kl(\hat \mu_a(t), \hat \mu^*(t)) + \log N_a(t)}$ & $\infty$
    \end{tabular}
    \end{center}
    \caption{
        \label{table:indexes}
        Examples of indexes.
    }
    \end{table}

The argument mostly follows the lines of UCB's; Hence the question is whether what are the properties that $I(-)$ must satisfy so that the ideas behind the local analysis of UCB still applies.
    The steps are as follows: (1) all arms are visited infinitely often; (2) visit rates converge; (3) at the asymptotic regime, if a draw of the bad arm yields maximal reward, it will be drawn again immediately; and (4) the third property is enough so that the index algorithm is subjected to poisoning. 
    By poisoning, we mean that if the bad arm provides maximal reward several times in a row, then whatever happens thereafter, the algorithm will keep picking the bad arm a few times in a row.

\subsection{Asymptotic Regimes}

Most of the regularity conditions that we require on $I(-)$ can be expressed in terms of continuity with respect to the topology of coordinate-wise equivalence of sequences. 
This topology appears naturally.
As times goes on, one may expect that $(\hat \mu_1(t), \hat \mu_2(t), N_1(t), N_2(t))$ gets closer an closer to $(\mu_1, \mu_2, n_1(t), n_2(t))$ where $n_1(t)$ and $n_2(t)$ are the deterministic visit rates of arms. 
In order to approximate $I(\hat \mu_2(t), \hat \mu_1(t), N_2(t), t)$ by $I(\mu_2, \mu_1, n_2(t), t)$ for instance, we need $I(-)$ to act continuously on equivalent sequences.  

    \begin{definition}[Asymptotic Topology]
    Consider a sequence $(x_1(n), \ldots, x_d(n))$ of $\R^d$. 
    A set $U \subseteq \N \to \R^d$ is said \emph{open at $x$} if there exists $\epsilon > 0$ such that it contains all $y : \N \to \R^d$ satisfying:
    $$
        \exists N, \forall n > N, \forall i:
        \quad
        \abs{x_i(n) - y_i(n)} < \epsilon \abs{x_i(n)}. 
    $$
    This topology that we obtain is the topology of coordinate-wise equivalence of sequences. 
    For instance, if $d = 1$, we have $x(n) \sim y(n)$ if, and only if $y$ belongs to all the neighborhoods of $x$; Hence we write $x \sim y$ if $y$ belongs to every neighborhood of $x$
    From now on, we endow the set of sequences of $\R^d$ with this topology. 
    \end{definition}

    Regarding the literature, it is fairly reasonable to have an index satisfying the following properties.
    (1) Monotonicity: the index is increasing in $\mu_a$, decreasing in $\mu_{-a}$, decreasing in $N_a$ and increasing in $t$. 
    (2) Unplayed arms see their index growing enough so that all arms are pulled infinitely often.
    (3) Convergence: an arm which is being pulled linearly often (e.g., an optimal arm) have converging index. 
    Most index algorithms in the literature can be reworked so that these properties are satisfied, see \cref{table:indexes}.

   \paragraph{Assumptions 1.}
    The first required set of assumptions is the following. 
    \begin{description}
        \item[\ASSUMPTION{1}] (Monotonicity)
            \textit{The index $I_a(-)$ is increasing in $\mu_a$, decreasing in $\mu_{-a}$, decreasing in $N_a$ and increasing in $t$. }

        \item[\ASSUMPTION{2}] (Diverging in $t$)
            \textit{For all fixed $n \ge 1$ and $\nu_2 \in [0, 1]$, in the neighborhood of $(\nu_2, \mu_1, n, t)$, $I_2(t) \to I_{\rm max}$.}

        \item[\ASSUMPTION{3}] (Convergence)
            \textit{ In the neighborhood of $(\mu_1, \mu_2, t, t)$, $I_1(t)$ converges to some positive $I(\mu_1, \mu_2) < I_{\rm max}$.}
    \end{description}

    \begin{lemma}
        \label{lemma:assumptions 1-3}
        Assume that $I(-)$ satisfies \ASSUMPTION{1-3}.
        Then, for $a =1,2$, $\hat \mu_a(t) \to \mu_a$ a.s. 
    \end{lemma}

    Equivalently, $t \mapsto (\hat \mu_1(t), \hat \mu_2(t))$ is in every neighborhood of $t \mapsto (\mu_1, \mu_2)$, i.e., the two sequences are topologically indistinguishable. 
    In practice, when running UCB, or KL-UCB or IMED, the arms' numbers of visits are such that all indexes are equal. 
    Because the index of the optimal arm converges to $I(\mu_1,\mu_2)$, $N_2(t)$ must be such that $I(\hat \mu_2(t), \hat \mu_1(t), N_2(t), t)$ is approximately $I(\mu_1,\mu_2)$, and the inverse must be continuous in $\hat \mu_2(t), \hat \mu_1(t), I(\mu_1,\mu_2)$. 
    This leads to the condition \ASSUMPTION{4}. 
    It is completed with \ASSUMPTION{5}, stating that the derivative of the inverse is not null. 
    The two combined guarantee that $N_2(t) \sim n_2(t)$ for some deterministic $n_2(t)$.
    The last condition \ASSUMPTION{6}, which is a formulation of the no-regret property, makes sure that $N_1(t) \sim t$ once $N_2(t) \sim n_2(t)$. 

    \paragraph{Assumptions 2.}
    Convergence of visit rates $N_a(t)$. 
    \begin{description}
        \item[\ASSUMPTION{4}] (Continuous inverse in $n$)
            \textit{Denote $f_{\nu_1, \nu_2, x}(t) := [I_2^{-1}(\nu_2, \nu_1, -, t)](x)$ the partial inverse in the number of visits for arm $a=2$ and let $n_2(t) := f_{\mu_1, \mu_2, I(\mu_1,\mu_2)}(t)$. The map $(t \mapsto (\nu_1(t), \nu_2(t), x(t))) \mapsto (t \mapsto f_{\nu_1(t), \nu_2(t), x(t)}(t))$ is continuous in a neighborhood of $(\mu_1, \mu_2, I(\mu_1,\mu_2))$.}

        \item[\ASSUMPTION{5}] (Asymptotic monotonicity in $n$)
            \textit{There is a non-negative definite function $\ell$ such that for $\epsilon > 0$ and in a neighborhood of $t \mapsto (\mu_1, \mu_2)$, 
            $$
                I(\nu_2, \nu_1, (1+\epsilon)n_2(t), t) \le (1 - \ell(\epsilon)) I(\nu_2, \nu_1, n_2(t), t),
            $$
            and similarly, $I(\nu_2, \nu_1, (1-\epsilon) n_2(t), t) \ge (1 + \ell(\epsilon)) I(\nu_2, \nu_1, n_2(t), t)$.
            }
        \item[\ASSUMPTION{6}] (No-Regret) \textit{$n_2(t)$ is sublinear in $t$, $n_2(t) \to \infty$ and $n_2(t) \sim n_2(at)$ (for all $a > 0$) when $t \to \infty$}.
    \end{description}

    \begin{lemma}
        \label{lemma:index asymptotic regime}
        If $I(-)$ satisfies \ASSUMPTION{1-6}, then $(\hat \mu_1(t), \hat \mu_2(t), N_1(t), N_2(t)) \sim (\mu_1, \mu_2, t, n_2(t))$ a.s. 
        The sequence $t \mapsto (\mu_1, \mu_2, t, n_2(t))$ will be called the \emph{asymptotic regime}.
    \end{lemma}

\subsection{Local Behavior in the Asymptotic Regime}

To analyze the local evolution of indexes in the asymptotic regime, we assume that for every arm, $I_a(t+h) - I_a(t)$ can be approximated by its Taylor expansion, and that this Taylor expansion depends continuously on the parameters $\hat \mu_a(t), N_a(t)$ and $t$. 
This is expressed by \ASSUMPTION{7}.
\ASSUMPTION{9} states that not all terms vary at the same speed; Namely, that the partial derivatives of $I_1(t)$ are negligible, and that in the Taylor expansion of $I_2(t+h) - I_2(t)$, the term $\partial_t I_2(t)$ can be neglected.
Lastly, \ASSUMPTION{8} states that the evolution of $I_2(t)$ relatively to $N_2(t)$ and $\hat \mu_2(t)$ are comparable, and the evolution relatively to $\hat \mu_2(t)$ is large enough in front of the one relatively to $N_2(t)$.
This guarantees that if the suboptimal arm $a=2$ is pulled and yield maximal reward $R_t=1$, it will be pulled in the next round. 

\paragraph{Assumptions 3.}
    Local properties of $I_a(t)$ in the asymptotic regime. 
    \begin{description}
        \item[\ASSUMPTION{7}] (Taylor expansion)
            \textit{
                In a neighborhood of the asymptotic regime (say $(\nu_1, \nu_2, m_1, m_2)$ in a neighborhood of $(\mu_1, \mu_2, n_1, n_2)$), for all fixed $h \ge 1$ and all arm $a$, we have:
                \begin{align*}
                    I_a(t+h) - I_a(t) \sim {} 
                    & (\nu_a(t+h) - \nu_a(t)) \cdot \partial_{\mu_a} I_a(t) \\
                    & + (\nu_{-a}(t+h) - \nu_{-a}(t)) \cdot \partial_{\mu_{-a}} I_a(t) \\
                    & + (m_a(t+h) - m_a(t)) \cdot \partial_n I_a(t) \\
                    & + h \cdot \partial_t I_a(t).
                \end{align*}
            }
        \item[\ASSUMPTION{8}]
            ($\rho$-optimism condition)
            \textit{
                There is a constant $\rho \in [0, 1)$, such that in a neighborhood of the asymptotic regime, $\partial_n I_2(t) \sim - \frac{\rho(1-\mu_2)}{m_2(t)} \partial_{\mu_2} I_2(t)$.
            }
        \item[\ASSUMPTION{9}] 
            (Negligible derivatives)
            \textit{
                In a neighborhood of the asymptotic regime, both $\partial_t I_1(t)$ and $\partial_t I_2(t)$ are $\oh(\partial_n I_2(t))$; And $\partial_{\mu_2} I_1(t) = o(\partial_{\mu_2} I_2(t))$.
            }
    \end{description}

    \begin{lemma}
        Let $I(-)$ an index satisfying \ASSUMPTION{1-9}.
        Fix $T > 0$. 
        There exists a sequence of events indexed by exploration episodes $(E_{\tau_k})$ with $\Pr(\liminf_k E_{\tau_k}) = 1$, such that, for all sequence $(U_t: t\ge1)$ of $\sigma(H_t)$-measurable events:
        $$
            \Pr \parens{\forall i < T: A_{\tau_k+i} = 2 \mid E_{\tau_k}, U_{\tau_k}}
            \ge 
            \mu_2^T.
        $$
    \end{lemma}
    \begin{proof}
        By Lemma~\ref{lemma:index asymptotic regime}, we know that $(\hat \mu_1(t), \hat \mu_2(t), N_1(t), N_2(t))$ goes to the asymptotic regime $(\mu_1, \mu_2, t, n_2(t))$ almost surely, so \ASSUMPTION{7-9} can be instantiated to the random quantities. 
    Suppose that $t$ is large enough and is such that over the time-range $\set{t, \ldots, t+h-1}$, we have $A_s = 2$. 
    Then we can write:
    \begin{align*}
        & \parens{I_2(t+h) - I_1(t)} - \parens{I_2(t) - I_1(t)}
        \\
        & \sim 
        \frac{\sum_{i=0}^{h-1}(R_{t+i} - \mu_2)}{N_2(t)} \parens{\partial_{\mu_2} I_2(t) - \partial_{\mu_2} I_1(t)}
        + h \partial_n I_2(t) 
        + h \parens{\partial_t I_2(t) - \partial I_1(t)}
        & \ASSUMPTION{7}
        \\
        & \sim
        \frac{\sum_{i=0}^{h-1}(R_{t+i} - \mu_2)}{N/_2(t)} \partial_{\mu_2} I_2(t)
        + h \partial_n I_2(t) 
        + h \parens{\partial_t I_2(t) - \partial I_1(t)}
        & \ASSUMPTION{9}
        \\
        & \gtrsim
        \frac{\sum_{i=0}^{h-1}(R_{t+i} - \mu_2) - \rho(1- \mu_2)h}{N_2(t)} \partial_{\mu_2} I_2(t)
        + h \parens{\partial_t I_2(t) - \partial I_1(t)}
        .
        & \ASSUMPTION{8}
    \end{align*}
    Assume that $R_{t+i} = 1$ for all $i \in \set{0, \ldots, h-1}$. 
    We get
    \begin{align*}
        \parens{I_2(t+h) - I_1(t)} - \parens{I_2(t) - I_1(t)}
        & \gtrsim
        \frac{(1-\rho)(1- \mu_2)h}{N_2(t)} \partial_{\mu_2} I_2(t)
        + h \parens{\partial_t I_2(t) - \partial I_1(t)}
        \\
        & \sim
        \frac{(1-\rho)(1- \mu_2)h}{N_2(t)} \partial_{\mu_2} I_2(t)
        .
        & \ASSUMPTION{9}
    \end{align*}
    Since $A_t = 2$, we have $I_2(t) - I_1(t) \ge 0$. 
    By \ASSUMPTION{1}, $\partial_{\mu_2}(I_2(t)) > 0$ so $I_2(t+h) - I_1(t+h) > 0$, hence $A_{t+h} = 2$. 
        We have established that, in the asymptotic regime, if $A_s = 2$ for $s \in \set{t, \ldots, t+h-1}$ with $R_s=1$, then $A_{t+h}=2$ as well.
        This means that the index policy essentially behaves like UCB: 
        If the bad arm only yields optimal rewards, it is repeatedly pulled.
    \end{proof}

    It means that Lemma~\ref{lemma:ucb non agnosticity} extends to general indexes satisfying \ASSUMPTION{1-9}.
    Therefore, and with the same proof, so does \cref{theorem:ucb sliding}: Index policies pull the bad arm for arbitrary long time-window infinitely often.
    \cref{theorem:ucb regret of exploration} also generalizes, and the regret of an index policy at exploration episodes can be predicted. 
    It locally behaves like a random walk. 
    The accuracy of the prediction is experimentally measured in \cref{figure:regret of exploration}.

    \begin{theorem}
        \label{theorem:general indexes}
        Let $I(-)$ an index satisfying \ASSUMPTION{1-9}.
        Then:
        \begin{itemize}
            \item[(1)] {\rm Sliding Regret:}
                $\SliReg(I; T) = (\mu_1-\mu_2)T$.
            \item[(2)] 
                {\rm Regret of Exploration:}
                Let $(X_t: t\ge 1)$ a sequence of i.i.d.~random variables with distribution ${\rm B}(\mu_2)$.
                Let $\sigma_T := T \wedge \inf\set{t\ge 1: - \rho(1-\mu_2) + \frac 1t \sum_{i=1}^t (X_t-\mu_2) \le 0}$.
                We have $\RegExp(I; T) \ge (\mu_1-\mu_2)\E[\sigma_T]$. 
        \end{itemize}
    \end{theorem}

    \subsection{Examples and Experiments}

    Checking that an index satisfies the requirements \ASSUMPTION{1-9} is mostly computations.
    \cref{example:check assumptions} details the checking process for IMED. 
    More examples are provided in \cref{table:asymptotic regimes}. 

    \begin{example}[IMED]
        \label{example:check assumptions}
        IMED from \cite{honda_non-asymptotic_2015} picks the arm maximizing:
        $$
            I_a(t) := \frac{\log(t)}{N_a(t)\kl(\hat \mu_a(t), \hat \mu^*(t)) + \log N_a(t)}.
        $$
        We have $I_{\rm max} = \infty$, and \ASSUMPTION{1-3} are obvious.
        When the arm $a=1$ is pulled linearly often, we have $I_1(t) = {\log(t)}/{\log N_1(t)} \sim 1$, so $I(\mu_1,\mu_2) = 1$.  
        We see that $n_2(t) := \frac{\log(t)}{\kl(\mu_2, \mu_1)}$, that depends continuously on $\mu_2, \mu_1$ so that \ASSUMPTION{4} is satisfied. 
        \ASSUMPTION{5-6} also immediately follow. 
        The last conditions are the ones that need more work, but they result from straight forward computations. 
        Asymptotically, for $\hat \mu_2 \equiv \hat \mu_2(t) < \hat \mu_1(t) \equiv \hat \mu_1$, we get:
        \begin{align*}
            I_1(t+h) - I_1(t) 
            & \sim \frac ht + \frac {N_1(t+h)-N_1(t)}{N_1(t)\log N_1(t)} \sim \frac ht,
            \\
            I_2(t+h) - I_2(t)
            & \sim 
            \frac{
                (\hat \mu_2(t+h) - \hat \mu_2(t))
                \log\parens{\frac{\hat \mu_1(1-\hat \mu_2)}{\hat \mu_2(1-\hat \mu_1)}}}{\kl(\hat \mu_2,\hat \mu_1)}
            - 
            \frac{N_2(t+h)-N_2(t)}{\frac{\log(t)}{\kl(\hat \mu_2,\hat \mu_1)}}
            + 
            \frac h{t\log(t)}.
        \end{align*}
        These two Taylor expansions are continuous in $\hat \mu_2$ and $\hat \mu_1$, so that \ASSUMPTION{5} is satisfied. 
        \ASSUMPTION{9} follows directly and \ASSUMPTION{8} can be checked numerically. 
        Following \cref{theorem:general indexes},
        $$
            \SliReg({\rm IMED}; T) = (\mu_1-\mu_2)T
        $$
        and its regret of exploration can be predicted via the random walk specified in \cref{theorem:general indexes}.2. 
    \end{example}

    \begin{table}
    \resizebox{\linewidth}{!}{
    \begin{tabular}{c|c|ccc}
        Algorithm & Index & $n_2(t)$ & $\partial_{\mu_2} I_2$ & $\partial_n I_2$ 
        \\ \hline
        UCB 
        & $\hat \mu_a(t) + \sqrt{\frac{2\log(t)}{N_a(t)}}$ 
        & $\frac{2\log(t)}{(\mu_1-\mu_2)^2}$
        & $1$
        & $- \frac{\mu_1 - \mu_2}{2n_2(t)}$
        \\
        MOSS 
        & $\hat \mu_a(t) + \sqrt{\frac{\log\parens{\frac{t}{2N_a(t)}}}{N_a(t)}}$
        & $\frac{\log(t)}{(\mu_1-\mu_2)^2}$
        & $1$
        & $- \frac{\mu_1 - \mu_2}{2n_2(t)}$
        \\ \hline
        UCB-V
        & $\hat \mu_a(t) + \sqrt{\frac{2\hat \mu_a(t)(1-\hat\mu_a(t))\log(t)}{N_a(t)}} + \frac{3c\log(t)}{N_a(t)}$
        & $\frac{2\mu_2(1-\mu_2)}{(\mu_1-\mu_2)^2}\parens{1+\sqrt{1+\frac{6c(\mu_1-\mu_2)}{\mu_2(1-\mu_2)}}}^2\log(t)$
        & $(*)$
        & $(**)$
        \\ \hline
        KL-UCB 
        & $\max\set{\mu : N_a(t) \kl(\hat \mu_a(t), \mu) \le \log(t)}$
        & $\frac{\log(t)}{\kl(\mu_2, \mu_1)}$
        & $\frac{\log\parens{\frac{\mu_1(1-\mu_2)}{\mu_2(1-\mu_1)}}}{\frac{\mu_1-\mu_2}{\mu_1(1-\mu_1)}}$
        & $-\frac{\kl(\mu_2, \mu_1)}{n_2(t)\frac{\mu_1-\mu_2}{\mu_1(1-\mu_1)}}$
        \\
        IMED 
        & $\frac{\log(t)}{N_a(t)\kl(\hat \mu_a(t), \hat \mu^*(t)) + \log N_a(t)}$
        & $\frac{\log(t)}{\kl(\mu_2,\mu_1)}$ 
        & $\frac{\log\parens{\frac{\mu_1(1-\mu_2)}{\mu_2(1-\mu_1)}}}{\kl(\mu_2,\mu_1)}$
        & $- \frac 1{n_2(t)}$
    \end{tabular}}
    \caption{
        \label{table:asymptotic regimes}
        Examples of asymptotic regimes.
        The missing entries of UCB-V are $(*)~\partial \mu_2 I_2 := 1 + (1-2\mu_2)\sqrt{\frac{\log(t)}{2n_2(t)\mu_2(1-\mu_2)}}$ and $(**)~\partial_n I_2 := - \frac {\log(t)}{n_2(t)} \parens{\sqrt{\frac{\mu_2(1-\mu_2)\log(t)}{2n_2(t)}} + \frac {6c\log(t)}{n_2(t)}}$.
        In this array, we can group algorithms in three families of algorithms with similar asymptotic regimes and identical ratios $n_2(t) \partial_{\mu_2} I_2/\partial_n I_2$, known to account for the regret of exploration: UCB and MOSS, UCB-V, KL-UCB and IMED.
    }
    \end{table}

    \begin{example}[Experiments]
        
        We extend the experiment of \cref{figure:regret of exploration} to other index policies. 
        We estimate the function $\RegExp'(t;T) := \E[\Reg(t; t+T)| \exists k, t = \tau_k]$ as a function of $t$. 
        To estimate this function, we repeatedly run the algorithm to obtain a family $S$ of samples of $(\tau_k, \Reg(\tau_k;\tau_k+T))$.
        Then, we approximate $\RegExp'(t; T)$ as the averages of $y$ for $(x,y) \in S$ such that $\abs{x - t} < W$ where $W$ is a parameter of the approximation. 
        In the experiments, we take $W = 128$ and $T = 100$.

        \begin{figure}
            \centering
            \includegraphics[width=.49\linewidth]{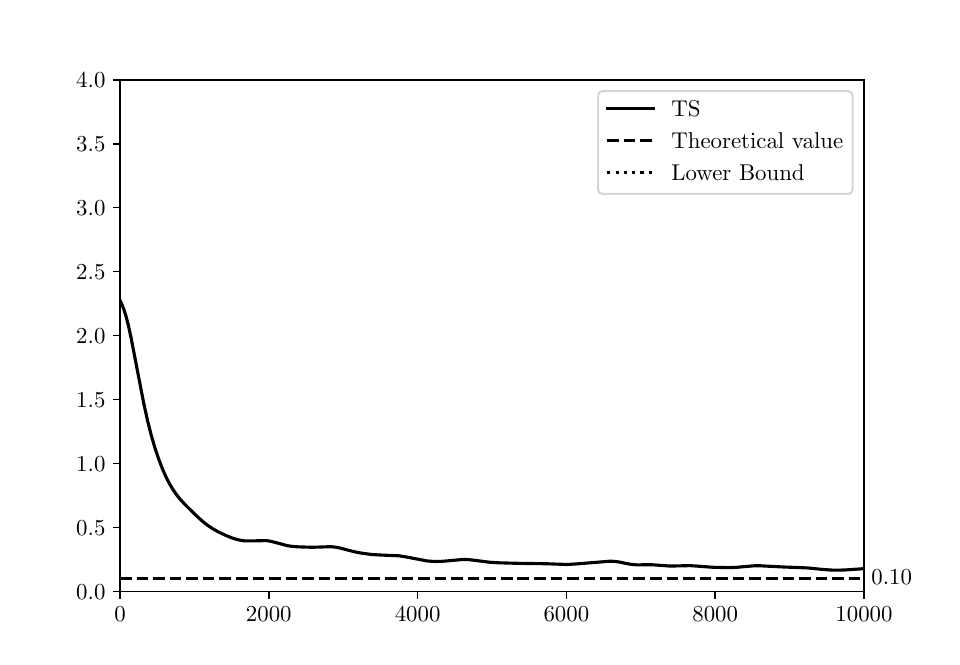}
            \includegraphics[width=.49\linewidth]{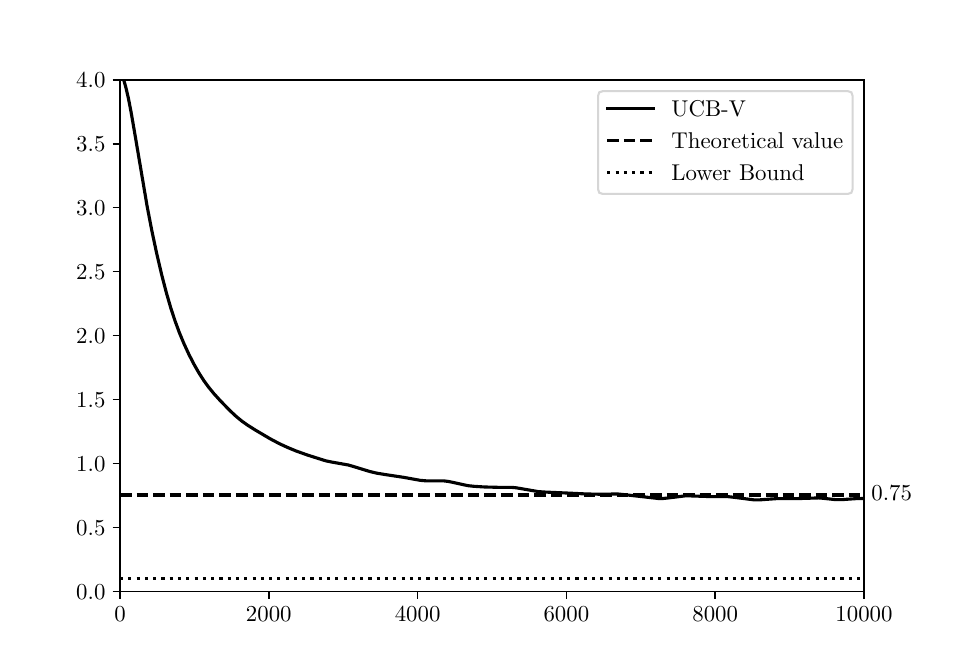}
            \includegraphics[width=.49\linewidth]{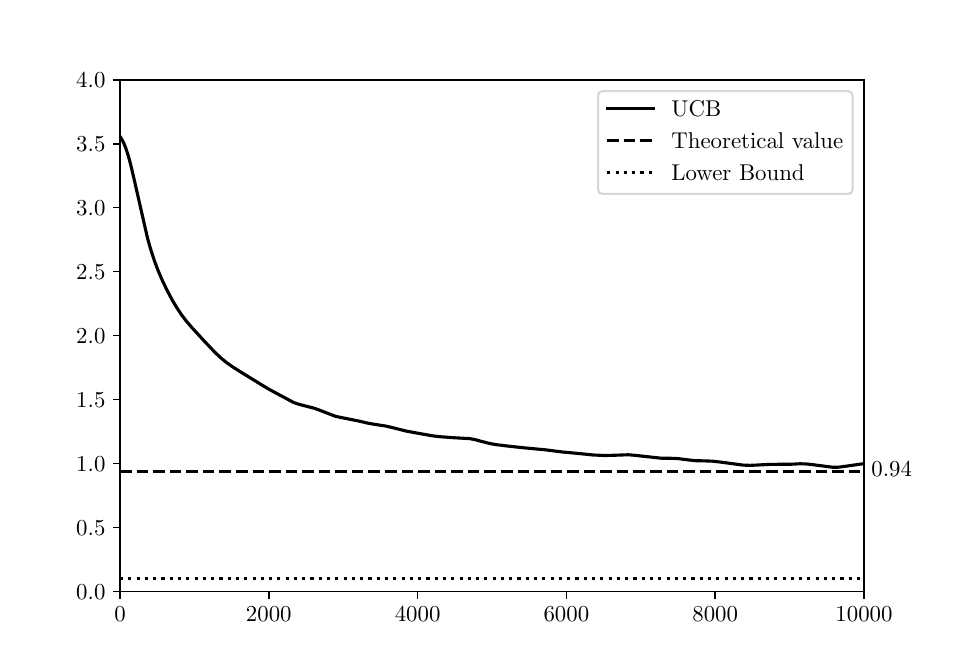}
            \includegraphics[width=.49\linewidth]{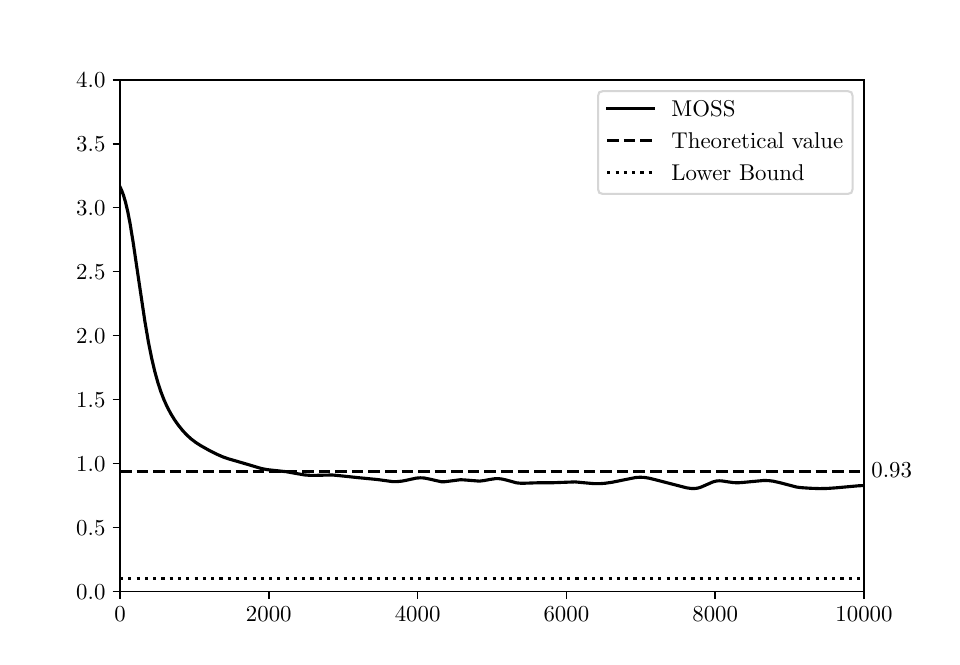}
            \includegraphics[width=.49\linewidth]{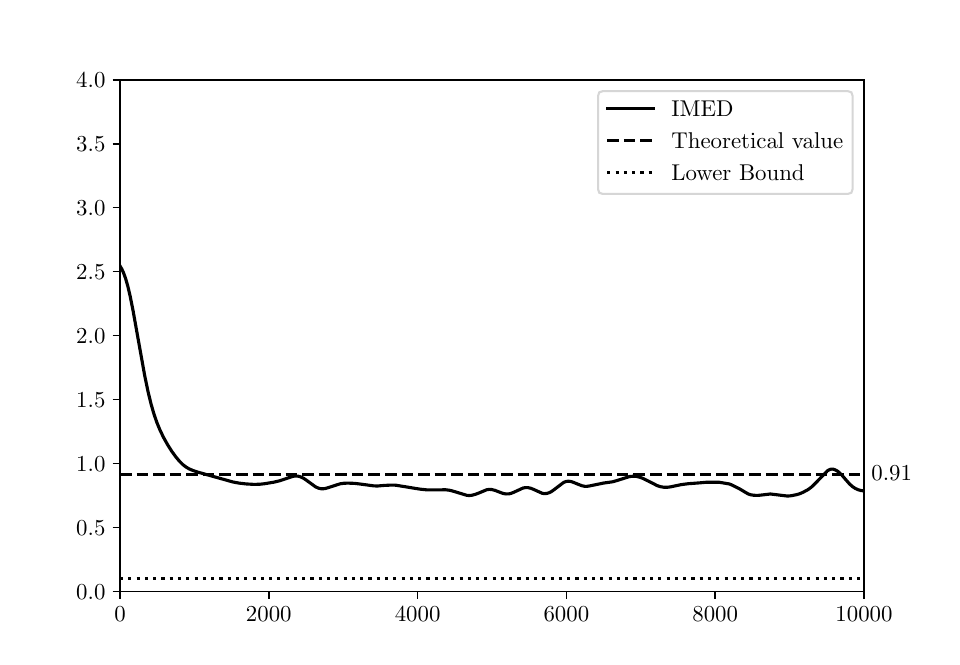}
            \includegraphics[width=.49\linewidth]{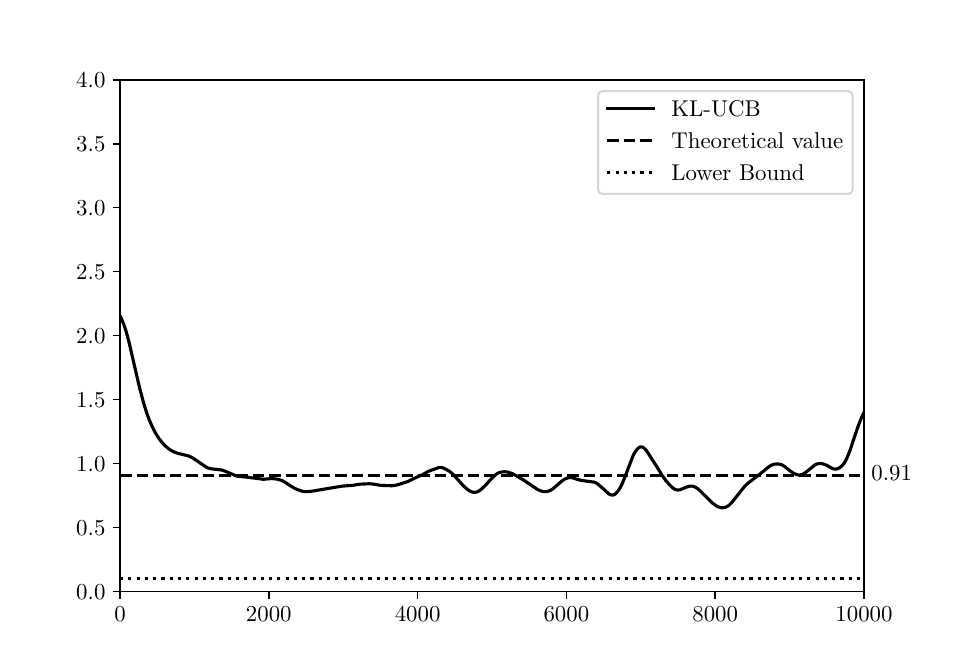}
            \caption{
                Estimated regret of exploration for various algorithms on the two-arm Bernoulli bandit $({\rm B}(0.9), {\rm B}(0.8))$. 
                The performance of TS and UCB from \cref{figure:regret of exploration} are compared to UCB-V, MOSS, IMED and KL-UCB. 
                The lower bound (dotted) of $\RegExp(-)$ is $0.1$. The theoretical value is reported to the right.
            }
        \end{figure}
        
        We overall observe a convergence to the predicted theoretical value (\cref{theorem:general indexes}.2). 
        Observing the precise rate of convergence of $\RegExp'(t; T)$ as a function of $t$ is rather difficult, especially for IMED and KL-UCB, because these algorithms are very aggressive and rarely pick the suboptimal arm, meaning that there a only a few exploration episodes during a run.
        The amount of data required to accurately estimate the curve increases exponentially with $t$. 
        Nonetheless, it seems that $\RegExp'(t; T)$ is slightly below the theoretical $(\mu_1-\mu_2)\E[\sigma_T]$ sometimes, see IMED for instance.
        This is due to two things.
        First, although we eventually have $|N_2(t) - n_2(t)| < \epsilon n_2(t)$ with $\epsilon$ as small as desired, for $t = 10000$, the correct $\epsilon$ may remain large.
        For instance, in IMED's index, the term $\log n_2(t)$ cannot be neglected in front of $n_2(t) \kl(\mu_2, \mu_1)$ even when $t = 10000$, implying that $N_2(t)$ and $n_2(t)$ are of the same order but still a bit far away.
        Second, the analysis assumes that the partial derivatives of the index stay approximately the same over $[t; t+T]$, which is quite imprecise when $t$ isn't large enough in front of $T$. 
    \end{example}

\clearpage
\section{Conclusion and Further Work}

In this paper, we design a new learning metric that allows to discriminate algorithms with comparable regrets: the sliding regret, that accounts for the local temporal behavior of the algorithm. 
We show that index algorithms (UCB, KL-UCB, IMED, etc.) inevitably have linear sliding regret while randomized algorithms such as TS and MED do not suffer from this issue.
The analysis of index algorithms underlines that such algorithms tend to locally behave poorly when exploration starts, i.e., when they switch to picking a suboptimal arm after a period of time they only picked optimal ones. 
This is why we introduce the regret of exploration, that quantifies the local regret at these critical time instants.
We show that index algorithms have suboptimal regret of exploration, in opposition to TS and MED. 

The study of the local temporal behavior of learning algorithms is far from over. 
While our results on the sliding regret are pretty conclusive, our analysis of the regret of exploration is still incomplete. 
Is our lower bound on the regret of exploration tight? 
While our experiment suggests that the answer is positive, this direction is yet to investigate. 
Another worthy direction is the analysis of EXP3 (\citealt{auer_gambling_1995}), whose behavior cannot be explained by our current proof techniques. 
When running EXP3, we observe that the one-shot pseudo-regret performance fit in between the dichotomous typical portraits pictured in \cref{figure:typical}. 
Hence, randomized algorithms do not typically behave like TS or MED in general. 

Another perspective is to extend this work beyond stochastic bandits.
While our proofs can be adapted to cover multi-arm bandits with non-Bernoulli reward distributions, they are specific to stochastic bandits. 
For instance, \cite{boone_regret_2023} initiated the design of the regret of exploration on Markov Decision Processes (MDP), but many things are still far from being understood. 
As a matter of fact, no algorithm with sublinear regret of exploration on general finite MDPs, nor ergodic MDPs, is known.
As far as the sliding regret is concerned, there is no existent result for MDPs.






\newpage

\appendix
\section{Almost-sure Properties of Consistent Algorithms}

\label{section:consistent algorithms}

This section is dedicated to the proof of the following result.

\begin{proposition}
    \label{proposition:consistent policies}
    Consider a policy such that whatever the distributions ${\rm F}$ on arms, the expected regret grows linearly, i.e., $\E_{\rm F}[\Reg(T)] = \oh(T)$.
    Then all arms are visited infinitely often, that is, $\Pr(\forall n, \exists t: N_a(t) \ge n) = 1$.
\end{proposition}
\begin{proof}
\STEP{1}
    Assume on the contrary that, for some distributions ${\rm F}$ on arms and for some arm $a$, $\Pr_{\rm F}(\forall t: N_a(t) < n) > 0$ where $n \ge 1$.  
    Because the expected regret is sublinear, $a$ has to be suboptimal. 
    Let ${\rm F}'$ any distribution on arms making $a$ the unique optimal arm, and such that ${\rm F}(a') = {\rm F}'(a')$ for all $a' \ne a$. 
    Denote the likelihood-ratio of the observations $(A_1, R_1, \ldots, A_{t-1}, R_{t-1})$ as 
    $$
        L_t \equiv L(A_1,R_1, \ldots, A_{t-1}, R_{t-1}) 
        := 
        \sum_b \sum_{s=1}^{t-1} 
        \indicator{A_s = b} 
        \frac{f_b(R_s)}{f'_b(R_s)}
    $$
    Denoting $\Fc_t := \sigma(A_1, R_1, \ldots, A_{t-1}, R_{t-1})$, it is known (see \citealt[Lemma~18]{kaufmann_complexity_2016}) that if $E$ is a $\Fc_t$-measurable event, then
    $$
        \Pr_{\rm F'}(E) = \E_{\rm F}\brackets{
            \indicator{E}
            \exp(-L_t)
        }.
    $$

    \STEP{2}
    Because ${\rm F}$ is non-degenerate with $0 < \mu_a < \mu^* < 1$, we can assume that $\mu'_a < 1$ and that there exists $c > 0$ such that for $r \in \set{0, 1}$, we have $\frac 1c \le \exp(f_a(r)/f'_a(r)) \le c$. 
    Then, 
    \begin{align*}
        \Pr_{\rm F'}
        \parens{
            \forall t: N_a(t) < n
        }
        & = \lim_{t \to \infty}
        \Pr_{\rm F'}
        \parens{
            N_a(t) < n
        }
        \\
        & = \lim_{t \to \infty}
        \E_{\rm F} \brackets{
            \indicator{N_a(t) < n}
            \exp\parens{
                -
                \sum_b
                \sum_{s=1}^{t-1}
                \indicator{A_s = b}
                \frac{f_b(R_s)}{f'_b(R_s)}
            }
        }
        \\
        & = \lim_{t \to \infty}
        \E_{\rm F} \brackets{
            \indicator{N_a(t) < n}
            \prod_{s=1}^{t-1} 
            \exp\parens{
                -
                \indicator{A_s = b}
                \frac{f_a(R_s)}{f'_a(R_s)}
            }
        }
        \\
        & \ge \lim_{t \to \infty}
        \E_{\rm F} \brackets{
            \indicator{N_a(t) < n}
            \parens{ \frac 1 c }^{N_a(t)}
        }
        \\
        & \ge c^{-n} > 0.     
    \end{align*}

    \STEP{3}
    Let $\Delta' > 0$ the gap between the optimal arm and the best suboptimal arm under ${\rm F'}$.
    We get
    \begin{align*}
        \E_{\rm F'}[N_a(T)] 
        &
        \le n \Pr_{\rm F'}(N_a(T) < n) + T \Pr_{\rm F'}(N_a(T) \ge n)
        \le n + T (1 - c^{-n}). 
    \end{align*}
    So:
    \begin{align*}
        \E_{\rm F'}[\Reg(T)]
        & \ge \Delta' \parens{
            T - \E_{\rm F'}[N_a(T)]
        }
        \\
        & \ge \Delta' (c^{-n} T - n) = \Omega(T).
    \end{align*}
    This comes in contradiction with $\E_{\rm F'}[\Reg(T)] = o(T)$. 
\end{proof}

\section{Analysis of Thompson Sampling}

\subsection{Preliminaries: Sanov's Theorem}

Our analysis of Thompson Sampling relies on a quantitative version of Sanov's Theorem. 

\begin{lemma}[Sanov's Theorem]
    \label{lemma:sanov}
    Let $q \in (0, 1)$ and $(X_n:n\ge 1)$ a family of i.i.d.~random variables with distribution ${\rm B}(q)$.
    Let $S_n := X_1 + \ldots + X_n$ and denote $\Pr_q(S_n \in \ldots)$ the induced probability distribution.
    Then, for $\epsilon > \frac 1n$, 
    \begin{eqnarray}
        \label{equation:sanov lower}
        \tfrac 1{n+1} e^{-n\kl(q-\epsilon-\frac 1n, q)} 
        &\le~~
        \Pr_q(S_n \le n(q-\epsilon))
        &\le~~ 
        n e^{-n\kl\parens{q-\epsilon, q}},
        \\
        \label{equation:sanov upper}
        \tfrac 1{n+1} e^{-n\kl\parens{q-\epsilon, q}} 
        &\le~~
        \Pr_q(S_n \ge n(q+\epsilon))
        &\le~~
        n e^{-n\kl\parens{q-\epsilon+\frac 1n, q}}.
    \end{eqnarray}
\end{lemma}
\begin{proof}
    Naming these inequalities ``Sanov's Theorem'' is a bit of an overstatement but is nonetheless very close to the original. 
    The proof is classic, but we write it below for the paper to be self-contained. 

\STEP{1}
        We start by a combinatorial lemma.
        Write $h(p) := - p \log(p) - (1 - p)\log(1-p)$ the Shannon entropy. 
        For all $n$ and $k \in \set{0, \ldots, n}$, we have
        $$
            \frac{e^{n h\parens{\frac kn}}}{n+1}
            \le 
            \binom nk
            \le
            e^{n h\parens{\frac kn}}. 
        $$
        To establish this, remark that $1 = \sum_{\ell} \binom n\ell (\frac kn)^\ell(1 - \frac kn)^{n - \ell}$. 
        The term for $\ell = k$ is equal to $e^{-n h(k/n)}$.
        In particular, we have $1 \ge \binom nk e^{-nh(k/n)}$, giving the upper bound above.
        But also, since the term for $\ell = k$ is the largest of the sum, we get $1 \le (n+1) \binom nk e^{-nh(k/n)}$, leading to the lower bound. 

    \STEP{2}
    Let $k \in \set{0, \ldots, n}$.
    We have
    \begin{align*}
        \Pr_q(S_n = k)
        & = \binom nk q^k(1-q)^{n-k} \\
        & = \binom nk \parens{q^{\frac kn}(1-q)^{1 - \frac kn}}^n \\
        & = \binom nk e^{-n h\parens{\frac kn}} \cdot e^{-n \kl\parens{\frac kn, q}}. 
    \end{align*}
    We therefore obtain
    $$
        \frac{e^{-n \kl\parens{\frac kn, q}}}{n+1}
        \le
        \Pr_q(S_n = k)
        \le 
        e^{-n \kl\parens{\frac kn, q}}. 
    $$

    \STEP{3}
    We establish the bounds for $\Pr_q(S_n \le n(q-\epsilon))$, see \eqref{equation:sanov lower}.
    $$
        \Pr_q(S_n \le n(q-\epsilon)) 
        = \sum\nolimits_{k = 0}^{\lfloor n(q-\epsilon)\rfloor} \Pr_q(S_n = k). 
    $$
    For the upper bound, remark that when over $\set{0, \ldots, \lfloor n(q-\epsilon)\rfloor}$, the function $k \mapsto e^{-n \kl\parens{\frac kn, q}}$ is decreasing, thus we get 
    $$
        \Pr_q(S_n \le n(q-\epsilon)) 
        \le \parens{n - \lfloor n(q-\epsilon)\rfloor} e^{-n\kl\parens{\frac{\lfloor n(q-\epsilon)\rfloor}{n}, q}}
        \le n e^{-n (\kl(q-\epsilon, q))}.
    $$
    Rearranging terms provides the upper bound part in Sanov's Theorem.
    For the lower bound, check that 
    $$
        \Pr_q(S_n \le n(q-\epsilon)) 
        \ge \frac{e^{-n\kl\parens{\frac{\lfloor n(q-\epsilon)\rfloor}{n}, q}}}{n+1}
        \ge \frac{e^{-n\kl(q-\epsilon-\frac 1n, q)}}{n+1}. 
    $$
    The bounds for $\Pr_q(S_n \ge n(q+\epsilon))$, see \eqref{equation:sanov upper}, are established similarly. 
\end{proof}

\subsection{The Almost-Sure Asymptotic Behavior of Thompson Sampling}

Starting from this section, we assume throughout that the bandit model is $({\rm B}(\mu_1), {\rm B}(\mu_2))$ with non-degenerate means $0 < \mu_2 < \mu_1 < 1$. 
We first bound the sampling rates of TS. 

\begin{lemma}
    \label{lemma:ts sampling rates}
    There exist a positive definite function $c : \R_+ \to \R$ and a family $(n_\epsilon : \epsilon > 0)$ such that for all $\epsilon > 0$, there exists a sequence of events $(G_t^\epsilon)$ with $\liminf G_t^\epsilon$ a.s., and such that:
    $$
        e^{-N_2(t)(1+c(\epsilon))\kl(\mu_2, \mu_1)}
        \le 
        \E\brackets{
            \indicator{A_t=2}
            \mid
            G_t^\epsilon, 
            N_2(t), N_2(t) \ge n_\epsilon
        }
        \le 
        e^{-N_2(t)(1-c(\epsilon))\kl(\mu_2, \mu_1)}. 
    $$
\end{lemma}
\begin{proof}
    \STEP{1}
    Denote $F^{\rm Beta}_{\alpha, \beta}$ (respectively $F^{\rm Bin}_{n, p}$) the c.d.f.~of a Beta distribution ${\rm Beta}(\alpha, \beta)$ (respectively a Binomial distribution ${\rm Binom}(n, p)$). 
    Using the Beta-Binomial trick and \cref{lemma:sanov}, we obtain:
    \begin{align*}
        & \E\brackets{
            \indicator{\theta_a(t) \le \hat \mu_a(t) - \epsilon} 
            \mid \hat \mu_a(t), N_a(t)
        }
        \\
        & = 
        F^{\rm Beta}_{1+S_a(t), 1+N_a(t)-S_a(t)}(\hat \mu_a(t)-\epsilon)
        \\
        & = 
        1 - F^{\rm Bin}_{N_a(t)+1, \hat \mu_a(t)-\epsilon}(S_a(t))
        & (\text{Beta-Binomial Trick})
        \\
        & \le \E \brackets{
            (N_a(t)+1) e^{-(N_a(t) + 1) \kl\parens{\frac{S_a(t)}{N_a(t)+1}, \hat \mu_a(t) - \epsilon}}
            \mid \hat \mu_a(t), N_a(t)
        }
        & \text{(Sanov, Lemma~\ref{lemma:sanov})}
        \\
        & =
        (N_a(t)+1) e^{-(N_a(t) + 1) \kl\parens{\frac{S_a(t)}{N_a(t)+1}, \hat \mu_a(t) - \epsilon}}.
    \end{align*}
    We can similarly derive a bound for $\indicator{\theta_a(t) \ge \hat \mu_a(t)+\epsilon}$, showing that:
    $$
        \E\brackets{
            \indicator{\theta_a(t) \ge \hat \mu_a(t) + \epsilon} 
            \mid \hat \mu_a(t), N_a(t)
        }
        \ge 
        \frac{
            e^{-(N_a(t) + 1) \kl\parens{\frac{S_a(t)}{N_a(t)+1}, \hat \mu_a(t) + \epsilon}}
        }{
            N_a(t)+2
        }.
    $$

    \STEP{2}
    Introduce the events $F_t := (\abs{\hat \mu_1(t) - \mu_1} < \epsilon/3)$ and $E_t := (N_1(t) > t^b)$ that are such that both $\liminf F_t$ and $\liminf E_t$ are almost-sure (see \citealt{kaufmann_thompson_2012}, Proposition 1 for $E_t$). 
    We have
    \begin{align*}
        \Pr(\forall t, \exists s \ge t : \theta_1(s) \le \mu_1 - \epsilon) 
        & = \lim_{t \to \infty}
        \Pr \parens{
            \exists s \ge t:
            \theta_1(s) \le \mu_1 - \epsilon, F_s^\epsilon, E_s
        }
        \\
        & \le \lim_{t \to \infty}
        \sum\nolimits_{s \ge t}
        \Pr \parens{
            \theta_1(s) \le \hat \mu_1(s) - \tfrac{2\epsilon}3, 
            F_s^\epsilon, E_s
        }
        \\
        & \le \lim_{t \to \infty}
        \sum\nolimits_{s \ge t}
        \E \brackets{
            \indicator{ \theta_1(s) \le \hat \mu_1(s) - \tfrac{2\epsilon}3}
            \mid
            F_s^\epsilon, E_s, N_1(s), \hat \mu_1(s)
        }
        \\
        & \le \lim_{t \to \infty}
        \sum\nolimits_{s \ge t}
        (t^b+1)e^{-(t^b+1) \kl\parens{\mu_1-\tfrac \epsilon3, \mu_1-\tfrac{2\epsilon}3}}
        \\
        & = 0.
    \end{align*}
    Accordingly, $\Pr(\exists t, \forall s \ge t : \theta_1(s) > \mu_1 - \epsilon) = 1$.
    Similarly one can show that $\Pr(\exists t, \forall s \ge t: \theta_1(s) < \mu_1 + \epsilon) = 1$. 
 
    \STEP{3}
    Following \STEP{2}, we see that in the asymptotic regime, the suboptimal arm $a = 2$ cannot be picked unless $\theta_2(t) \ge \mu_1 - \epsilon$.
    And conversely, if $\theta_2(t) \ge \mu_1 + \epsilon$, then arm $2$ is pulled. 
    Introduce the asymptotically almost sure event:
    $$
        G^\epsilon_t := (\abs{\theta_1(t) - \mu_1} < \epsilon) 
        \cap (\abs{\hat \mu_2(t) - \mu_2} < \epsilon). 
    $$
    The probability to pick $A_t = 2$ conditionally on $G_t^\epsilon$ is bounded accordingly:
    $$
        \frac{e^{-(N_2(t)+1)\kl(\mu_2-\epsilon, \mu_1+\epsilon)}}{N_2(t)+2}
        \le 
        \E\brackets{
            \indicator{A_t=2}
            \mid G_t^\epsilon, N_2(t)
        }
        \le (N_2(t)+1) e^{-(N_2(t)+1)\kl(\mu_2+\epsilon, \mu_1-\epsilon)}. 
    $$
    Therefore, there exists a positive definite function $c(\epsilon)$ such that, when $N_2(t)$ is large enough relatively to $\epsilon$, say $N_2(t) \ge n_\epsilon$, we have:
    $$
        e^{-N_2(t)(1+c(\epsilon))\kl(\mu_2, \mu_1)}
        \le 
        \E\brackets{
            \indicator{A_t=2}
            \mid G_t^\epsilon, N_2(t), N_2(t) \ge n_\epsilon
        }
        \le 
        e^{-N_2(t)(1-c(\epsilon))\kl(\mu_2, \mu_1)},
    $$
    establishing the claim. 
\end{proof}

The second result of this section provides a precise description of TS's visit rates at infinity. 
The visit rate of the suboptimal, $N_2(t)$, will be later called the \emph{asymptotic regime} of TS. 

\begin{lemma}
    \label{lemma:ts almost sure rates}
    For all $\delta > 0$, 
    $$
        \Pr \parens{
            \exists t,
            \forall s \ge t:
            \abs{N_2(s) - \frac{\log(s)}{\kl(\mu_2,\mu_1)}}
            \le \delta \cdot \frac{\log(s)}{\kl(\mu_2, \mu_1)}
        }
        = 1.
    $$
\end{lemma}

\begin{proof}
    Let $\delta > 0$.
    Denote $c \equiv c(\epsilon)$ and $k_0 \equiv \kl(\mu_2, \mu_1)$ for short, and choose $\epsilon$ small enough so that $\frac 1{1+2c(\epsilon)} > 1-\delta$ and $3c(\epsilon) < \delta$. 

    \STEP{1}
    We show that the event $(N_2(t) \ge \frac{\log(t)}{(1+2c)k_0})$ holds eventually. 
    We proceed by considering the complementary event and denote $\lambda_t := \frac{\log(t)}{(1+2c)k_0}$.
    From Lemma~\ref{lemma:ts sampling rates} we also know that 
    $$
        \E\brackets{
            \indicator{A_s \ne 2} 
            \mid
            G_s^\epsilon, 
            n_\epsilon \le N_2(s) \le \lambda_t
        }
        \le 
        1 - e^{-\frac{(1+c)\log(t)}{1+2c}}.
    $$
    Denote $E_s^\epsilon := G_s^\epsilon \cap \parens{n_\epsilon \le N_2(s) \le \lambda_s}$ for short.
    Remark that if $N_2(s) < \lambda_s$, then the arm $a = 2$ has been sampled less than $\lambda_s$ times over $[\tfrac 12 s, s]$ with $N_2(u) < \lambda_s$ each time. 
    Said differently, there exists $\Lambda$ a subset of $[\tfrac 12s, s]$ with at least $\tfrac s2 - \lambda_s$ elements (we write $\Lambda \subseteq_{s/2 - \lambda_s} [\tfrac 12s, s]$) such that for all $i \in \Lambda$, $A_i = 1$.
    Therefore, where $F_j^\epsilon$ below is a shorthand for $(A_j=1, N_2(j) < \lambda_s, E_j^\epsilon)$, we have:
    \begin{align*}
        &\Pr\parens{
            \forall t, \exists s \ge t:
            N_2(s) < \lambda_s
        }
        \\
        & = \lim_{t \to \infty}
        \Pr \parens{
            \exists s \ge t:
            N_2(s) < \lambda_s
        }
        \\
        & \le 
        \lim_{t \to \infty}
        \Pr \parens{
            \exists s \ge t, \exists \Lambda \subseteq_{s/2 - \lambda_s} [\tfrac 12s, s],
            \forall i \in \Lambda:
            A_i = 1, N_2(i) < \lambda_s
        }
        \\
        & =
        \lim_{t \to \infty}
        \Pr \parens{
            \exists s \ge t, \exists \Lambda \subseteq_{s/2 - \lambda_s} [\tfrac 12s, s],
            \forall i \in \Lambda:
            A_i = 1, N_2(i) < \lambda_s, E_i^\epsilon
        }
        \\
        & = 
        \lim_{t \to \infty}
        \sum\nolimits_{s \ge t}
        \sum\nolimits_{I \subseteq_{s/2-\lambda_s} [\frac 12s, s]}
        \prod\nolimits_{i \in I}
        \Pr \parens{
            A_i = 1, N_2(i) < \lambda_s, E_i^\epsilon
            \mid \forall j < i \in I: F_j^\epsilon
        }
        \\
        & \le 
        \lim_{t \to \infty}
        \sum\nolimits_{s \ge t}
        \sum\nolimits_{I \subseteq_{s/2-\lambda_s} [\frac 12s, s]}
        \prod\nolimits_{i \in I}
        \Pr \parens{
            A_i = 1
            \mid N_2(i) < \lambda_s, E_i^\epsilon, (\forall j < i \in I: F_j^\epsilon)
        }
        \\
        & \le 
        \lim_{t \to \infty}
        \sum\nolimits_{s \ge t}
        \sum\nolimits_{I \subseteq_{s/2-\lambda_s} [\frac 12s, s]}
        \parens{1 - e^{- \frac{(1+c)\log(t)}{1+2c}}}^{\frac t2 - \frac{\log(t)}{(1+2c)k_0}}
        \\
        &\le
        \lim_{t \to \infty}
        \sum\nolimits_{s \ge t}
        \binom{t/2}{\tfrac{\log(t)}{(1+2c)k_0}}
        \parens{1 - e^{- \frac{(1+c)\log(t)}{1+2c}}}^{\frac t2 - \frac{\log(t)}{(1+2c)k_0}}.
    \end{align*}
    Using standard equivalents, the summand happens to be asymptotically upper bounded by
    \begin{align*}
        e^{C \log^2(t)}
        \cdot
        e^{-C t^{1 - \frac{1+c}{1+2c}}} 
        & \lesssim
        e^{-C' t^{\frac{c}{1+2c}}} 
    \end{align*}
    where $C > C' > 0$. 
    This term has finite sum. 
    We conclude has follows:
    \begin{align*}
        \Pr\parens{\forall t, \exists s \ge t: N_2(t) < \tfrac{\log(s)}{(1+2c)k_0}}
        & \le 
        \lim_{t \to \infty} \sum\nolimits_{s \ge t} e^{-C's \frac{c}{1+2c}} = 0.
    \end{align*}

    \STEP{2}
    We show that the event $(N_2(t) \le \frac{(1+3c)\log(t)}{k_0})$ holds eventually. 
    Again, we consider the complementary event and denote $\lambda_t := \frac{1}{k_0}\log(t)$. 
    By Lemma~\ref{lemma:ts sampling rates},
    $$
        \E \brackets{
            \indicator{A_s = 2}
            \mid 
            G_s^\epsilon, 
            n_\epsilon \le N_2(s),
            (1+2c)\lambda_t < N_2(s)
        }
        \le
        e^{-(1+c-2c^2)\log(t)}.
    $$
    Let $E_s^\epsilon := (G_s^\epsilon) \cap (n_\epsilon \le N_2(s)) \cap ((1+2c)\lambda_t < N_2(s))$ for short.
    Note that if $N_2(t) > (1+3c)\lambda_t$, then it has been sampled at least $c \lambda_t)$ times with $N_2(s) > (1+2c)\lambda_t$ over the time interval $[(1+2c)\lambda_t, t]$.
    So, there exists $\Lambda$ a subset of $[(1+2c)\lambda_t, t]$ of size at most $c\lambda_t$ (we write $\Lambda \subseteq_{c\lambda_t} [(1+2c)\lambda_t, t]$) such that for all $i \in \Lambda$, $A_i = 2$. 
    Therefore, and where $F_j^\epsilon$ below is a shorthand for $(A_j=2, N_2(j)>(1+2c)\lambda_j, E_j^\epsilon)$, we have
    \begin{align*}
        & \Pr \parens{
            \forall t, \exists s \ge t:
            N_2(s) > (1+3c)\lambda_s
        }
        \\
        & =
        \lim_{t \to \infty}
        \Pr \parens{
            \exists s \ge t:
            N_2(s) > (1+3c)\lambda_s
        }
        \\
        & \le
        \lim_{t \to \infty}
        \Pr \parens{
            \exists s \ge t,
            \exists \Lambda \subseteq_{c\lambda_s} [(1+2c)\lambda_s, s],
            \forall i \in \Lambda:
            A_i = 2,
            N_2(i) > (1+2c)\lambda_s
        }
        \\
        & =
        \lim_{t \to \infty}
        \Pr \parens{
            \exists s \ge t,
            \exists \Lambda \subseteq_{c\lambda_s} [(1+2c)\lambda_s, s],
            \forall i \in \Lambda:
            A_i = 2,
            N_2(i) > (1+2c)\lambda_s,
            E_i^\epsilon
        }
        \\
        & = 
        \lim_{t \to \infty}
        \sum\nolimits_{s \ge t}
        \sum\nolimits_{I \subseteq_{c\lambda_s}[(1+2c)\lambda_s, s]}
        \prod\nolimits_{i \in I}
        \Pr \parens{
            A_i = 2,
            N_2(i) > (1+2c)\lambda_s,
            E_i^\epsilon
            \mid 
            \forall j < i \in I: F_j^\epsilon
        }
        \\
        & = 
        \lim_{t \to \infty}
        \sum\nolimits_{s \ge t}
        \sum\nolimits_{I \subseteq_{c\lambda_s}[(1+2c)\lambda_s, s]}
        \prod\nolimits_{i \in I}
        \Pr \parens{
            A_i = 2
            \mid
            N_2(i) > (1+2c)\lambda_s,
            E_i^\epsilon,
            (\forall j < i \in I: F_j^\epsilon)
        }
        \\
        & \le
        \lim_{t \to \infty}
        \sum\nolimits_{s \ge t}
        \binom{s}{\tfrac{c\log(s)}{k_0}}
        e^{-(1+c-2c^2)\log(s)\cdot \frac{c}{k_0}\log(s)}.
    \end{align*}
    Using standard equivalents, the summand is asymptotically upper bounded by
    $$
        e^{\parens{\frac{c}{k_0} + o(1)}\log^2(t)}
        \cdot
        e^{-(1 + c - 2c^2)\frac{c}{k_0}\log^2(t)} 
        =
        e^{-(c - 2c^2 + o(1))\frac c{k_0}\log^2(t)}.
    $$
    Again, this has finite sum. 
    We conclude:
    \begin{align*}
        \Pr \parens{
            \forall t, \exists s \ge t:
            N_2(s) > (1+3c)\lambda_s
        }
        & = 0.
    \end{align*}
    This concludes the proof. 
\end{proof}

\subsection{Proof of \cref{theorem:ts sliding regret}}

\begin{proof}[Proof of \cref{theorem:ts sliding regret}]
    We conclude that Thompson Sampling has optimal sliding regret. 
    Fix $T \ge 1$.
    Combining Lemma~\ref{lemma:ts sampling rates} and Lemma~\ref{lemma:ts almost sure rates}, we see that for all $\epsilon > 0$, there exists a sequence of events $(E_t^\epsilon)$ with $\Pr(\liminf E_t^\epsilon)=1$, and such that:
    $$
        \Pr\parens{A_t = 2 \mid E_t^\epsilon} \le e^{-(1-\epsilon)\log(t)} = \frac 1{t^{1- \epsilon}}.
    $$
    Since all arms are visited infinitely often, eventually $T$ is negligible in front of $N_2(t)$, meaning that for all partial history $H_{t:t+h} = h_{t:t+h}$ over $[t, t+h]$ (with $h \le T$), we will have
    $$
        \Pr \parens{A_{t+h} = 2 \mid E_t^\epsilon, H_{t:t+h}=h_{t:t+h}} \le \frac 1{t^{1-\epsilon}}. 
    $$
    Conclude with \cref{theorem:agnosticity}.
\end{proof}

\section{Analysis of UCB}

\label{section:ucb analysis}

\subsection{The Asymptotic Regime of UCB}

\textbf{Proposition~\ref{proposition:asymptotic ucb}}
{\it 
    For all $\epsilon > 0$ and when running UCB, both of the following hold:
    \begin{itemize}
        \item[(1)] 
            $\Pr(\exists t, \forall s \ge t: \forall a, \abs{\mu_a(s) - \mu_a(s)} < \epsilon) = 1$;

        \item[(2)]
            $\Pr\parens{
                \exists t, \forall s \ge t:
                \abs{N_2(t) - 2\parens{\tfrac 1{\mu_1-\mu_2}}^2\log(t)}
                < \epsilon \cdot 2\parens{\tfrac 1{\mu_1-\mu_2}}^2\log(t)
            } = 1.$
    \end{itemize}
}

\begin{proof}[Proof of Proposition~\ref{proposition:asymptotic ucb}]
    Because UCB has sublinear expected regret, all arms are visited infinitely often by Proposition~\ref{proposition:consistent policies}, hence by the Strong Law of Large numbers, the empirical estimates of every arm converge to their true means. 
    This proves Proposition~\ref{proposition:asymptotic ucb}.1. 
    We will denote $E_t^\epsilon := (\forall a: |\hat \mu_a(t) - \mu_a| < \epsilon)$. 
    We now focus on the proof of Proposition~\ref{proposition:asymptotic ucb}.2.
    Denote $\lambda_t := \frac2{(\mu_1-\mu_2)^2}\log(t)$ the theoretical visit rate of arm $2$

    \STEP{1}
    Let $\epsilon > 0$. 
    We show that the event $(N_2(t) > (1-\epsilon)\lambda_t)$ holds eventually. 
    As usual, we proceed by considering the complementary event. 
    Let $\delta > 0$.
    Remark that if the arm $a=2$ has been visited less than $(1-\epsilon)\lambda_s$ times, then the other arm $a=1$ must have been pulled within the time range $\set{s - \lambda_s - 1, s}$, hence within $[\tfrac 12s, s]$ provided that $s$ is large enough. 
    Since $\lambda_s = o(s)$, we can in addition assume that when $a=1$ is pulled, $N_1(s) \ge \tfrac 12s$. 
    Accordingly, and denoting $F_u^\epsilon := (N_2(u) \le (1-\epsilon)\lambda_s) \cap (N_1(u) \ge \tfrac 12 s)$ for short, we have:
    \begin{align*}
        &\Pr\parens{
            \forall t, \exists s \ge t:
            N_2(s) \le (1-\epsilon)\lambda_s
        }
        \\
        & =
        \lim_{t \to \infty}
        \Pr \parens{
            \exists s \ge t:
            N_2(s) \le (1-\epsilon)\lambda_s,
        }
        \\
        & \le
        \lim_{t \to \infty}
        \Pr \parens{
            \exists s \ge t, \exists u \in [\tfrac 12s, s]:
            N_2(u) \le (1-\epsilon)\lambda_s,
            N_1(u) \ge \tfrac 12s, A_u = 1
        }
        \\
        & =
        \lim_{t \to \infty}
        \Pr \parens{
            \exists s \ge t, \exists u \in [\tfrac 12s, s]:
            F_u^\epsilon,
            E_u^\delta,
            \hat \mu_2(u) + \sqrt{\frac{2\log(u)}{N_2(u)}}
            \le
            \hat \mu_1(u) + \sqrt{\frac{2\log(u)}{N_1(u)}}
        }
        \\
        & \le
        \lim_{t \to \infty}
        \Pr \parens{
            \exists s \ge t, \exists u \in [\tfrac 12s, s]:
            \mu_2 - \delta + \sqrt{\frac{2\log(\tfrac 12s)}{\frac{2(1-\epsilon)}{(\mu_1-\mu_2)^2}\log(s)}}
            \le
            \mu_1 + \delta + \sqrt{\frac{2\log(s)}{\tfrac 12s}}
        }
        \\
        & \le
        \lim_{t \to \infty}
        \indicator{
            \frac{\mu_1-\mu_2}{\sqrt{1-\epsilon}} 
            \cdot
            \sqrt{\frac{\log(\tfrac 12t)}{\log(t)}}
            \le
            \mu_1 - \mu_2 + 3\delta
        }.
    \end{align*}
    In the above, $\delta > 0$ can be chosen arbitrarily small. 
    Since $\sqrt{1-\epsilon} < 1$, we see that by choosing $\delta$ small regarding $\epsilon$, we obtain $\Pr(\forall t, \exists s \ge t: N_2(s) \le (1-\epsilon)\lambda_s) = 0$. 

    \STEP{2}
    Let $\epsilon > 0$. 
    We now show that the event $(N_2(t) < (1+\epsilon)\lambda_t)$ holds eventually.
    Let $\delta > 0$.
    Observe that if $N_2(s) \ge (1+\epsilon)\lambda_s$, then arm $2$ must have been pulled within the time range $\set{(1+\epsilon)\lambda_s, \ldots, s}$ with $N_2(u) \ge (1+\epsilon)\lambda_s$.  
    Following this idea, we obtain:
    \begin{align*}
        & \Pr \parens{
            \forall t, \exists s \ge t:
            N_2(s) \ge (1+\epsilon)\lambda_s
        }
        \\
        & \le
        \lim_{t \to \infty}
        \Pr \parens{
            \exists s \ge t, \exists u \in [(1+\epsilon)\lambda_s, s]:
            N_2(u) \ge (1+\epsilon) \lambda_s, A_u = 2
        }
        \\
        & =
        \lim_{t \to \infty}
        \Pr \parens{
            \exists s \ge t, \exists u \in [(1+\epsilon)\lambda_s, s]:
            E_u^\delta,
            N_2(u) \ge (1+\epsilon) \lambda_s, A_u = 2
        }
        \\
        & \le
        \lim_{t \to \infty}
        \Pr \parens{
            \exists s \ge t, \exists u \in [(1+\epsilon)\lambda_s, s]:
            \mu_2 + \delta + (\mu_1-\mu_2)\sqrt{\frac{\log(u)}{(1+\epsilon)\log(s)}}
            \ge \mu_1 - \delta
        }
        \\
        & \le
        \lim_{t \to \infty}
        \indicator {
            \frac{\mu_1-\mu_2}{\sqrt{1+\epsilon}}
            \ge 
            \mu_1 - \mu_2 - 2 \delta
        }.
    \end{align*}
    The above indicator is asymptotically $0$ when $\delta$ is small enough. 
\end{proof}

\subsection{The Sliding Regret of UCB: Proof of Lemma~\ref{lemma:ucb non agnosticity}}

    As given by Proposition~\ref{proposition:asymptotic ucb}, the asymptotic regime is denoted 
    $$
        E^\epsilon_t 
        := 
        (\forall a: |\hat \mu_a(t) - \mu_a| < \epsilon)
        \cap 
        \parens{
            N_2(t) = \tfrac{2\cdot(1\pm\epsilon)}{(\mu_1-\mu_2)^2}\log(t)
        }
    $$
    Denote $I_a(t) := \hat \mu_a(t) + \sqrt{2\log(t)/N_a(t)}$ UCB's index of arm $a$. 

\begin{lemma}
    \label{lemma:ucb key deviations}
    Let $G_{t:t+h} := (\forall i < h, A_{t+i}=2)$.
    For all $\epsilon > 0$ and $h \ge 1$, there exists $\delta, T > 0$ such that, for all $t > T$ and on $E_t^\delta \cap G_{t:t+h}$, we have:
    $$
        \abs{
            \parens{I_1(t+h) - I_2(t+h)}
            -
            \parens{
                I_1(t) - I_2(t)
                - \frac{\sum_{i<h} (R_{t+i} - \mu_2 - \frac{\mu_1-\mu_2}2)}{N_2(t)}
            }
        }
        \le \frac{2h\epsilon}{N_2(t)}.
    $$
\end{lemma}

\begin{proof}
    Fix $\epsilon > 0$.
    (\textbf{STEP 1})
    The time variations of UCB's indexes are given by:
    $$
        I_a(t+h) - I_a(t)
        = 
        \parens{\hat\mu_a(t+h) - \hat\mu_a(t)}
        + 
        \parens{\sqrt{\frac{2\log(t+h)}{N_a(t+h)}} - \sqrt{\frac{2\log(t)}{N_a(t)}}}
    $$
    This is split into two terms. 
    There is the variation of the empirical estimate, and the variation of the optimistic bonus. 
    Considering arm $1$, since $N_1(t+h) = N_1(t)$ on $G_{t:t+h}$, we get, when $t\to\infty$, 
    \begin{align*}
        I_1(t+h) - I_1(t)
        & = 
        \sqrt{\frac 2{N_1(t)}}
        \parens{\sqrt{\log(t+h)} - \sqrt{\log(t)}}
        \\
        & \sim
        \frac{\frac ht}{\sqrt{2\log(t)N_1(t)}}
        = \oh\parens{\frac{h}{t\sqrt{N_1(t)}}}.
    \end{align*}
    This will appear to be negligible in comparison to $I_2(t+h) - I_2(t)$. 

    (\textbf{STEP 2})
    We know bound the variations of the empirical estimates of arm $2$. 
    We have:
    \begin{align*}
        \hat \mu_2(t+h) - \hat \mu_2(t)
        & = \frac{\sum_{i < h} R_{t+i} - h \hat \mu_2(t)}{N_2(t)+h}
        = \frac{\sum_{i<h} (R_{t+i}-\mu_2)}{N_2(t)+h} + \frac{h(\mu_2 - \hat \mu_2(t))}{N_2(t)+h}.
    \end{align*}
    Because arms are visited infinitely often, we have $\mu_2 - \mu_2(t) < \epsilon$ eventually, with $\epsilon > 0$ fixed. 
    Since, $N_2(t) + h \sim N_2(t)$, hence, when $t \to \infty$ and for $\delta > 0$ small regarding $\epsilon$, on $E_t^\delta \cap G_{t:t+h}$, we have:
    $$
        \abs{
            \hat \mu_2(t+h) - \hat \mu_2(t) 
            -
            \frac{\sum_{i<h} R_{t+i} - \mu_2}{N_2(t)}
        }
        \le 
        \frac{h\epsilon}{N_2(t)}
        \quad\text{a.s.}
    $$

    (\textbf{STEP 3})
    We now bound the variation of the optimistic bonus of arm $2$, 
    \begin{align*}
        \sqrt{\frac{2\log(t+h)}{N_2(t+h)}} - \sqrt{\frac{2\log(t)}{N_2(t)}}
        & = 
        \sqrt{\frac{2\log(t+h)}{N_2(t+h)}} - \sqrt{\frac{2\log(t)}{N_2(t+h)}}
        + 
        \sqrt{\frac{2\log(t)}{N_2(t+h)}} - \sqrt{\frac{2\log(t)}{N_2(t)}}
        \\
        & \sim 
        \oh \parens{\frac h{tN_2(t)}} +
        \frac{h}{2N_2(t)}\sqrt{\frac{2\log(t)}{N_2(t)}}.
    \end{align*}
    Provided that $\delta > 0$ is small enough, we have on $E_t^\delta$:
    $$
        \abs{
            \sqrt{\frac{2\log(t+h)}{N_2(t+h)}} - \sqrt{\frac{2\log(t)}{N_2(t)}}
            -
            \frac{h(\mu_1 - \mu_2)}{2N_2(t)}
        }
        \le
        \frac{h\epsilon}{N_2(t)}
    $$

    (\textbf{STEP 4})
    All together, on $E_t^\delta \cap G_{t:t+h}$, we have:
    $$
        I_1(t+h) - I_2(t+h)
        =
        I_1(t) - I_2(t)
        - \frac{\sum_{i<h} (R_{t+i} - \mu_2 - \frac{\mu_1-\mu_2}2)}{N_2(t)}
        \pm \frac{2h\epsilon}{N_2(t)}.
    $$
    This proves the claim. 
\end{proof}

We prove Lemma~\ref{lemma:ucb non agnosticity} as a corollary below. 
\bigskip

\noindent\textbf{Lemma~\ref{lemma:ucb non agnosticity}}
{
    \it
    Consider running UCB, and fix $h > 0$. 
    There exists a sequence of events indexed by exploration episodes $(E_{\tau_k})$ with $\Pr(\liminf_k E_{\tau_k}) = 1$, such that, for all sequence $(U_t: t\ge1)$ of $\sigma(H_t)$-measurable events:
$$
    \Pr \parens{\forall i < h: A_{\tau_k+i} = 2 \mid E_{\tau_k}, U_{\tau_k}}
    \ge 
    \mu_2^h.
$$
}
\begin{proof}
    Fix $h > 0$ and let $\delta(h), T(h) > 0$ as given by Lemma~\ref{lemma:ucb key deviations} for some arbitrary $\epsilon > 0$. 
    Let $G'_{t:t+h} := (\forall i < h, R_{t+i}=1)$, stating that every arm pull over $[t, t+h)$ provides full reward.
    On $E_t^{\delta(h)} \cap G_{t:t+h} \cap G'_{t:t+h}$ with $t \ge T(h)$, we have:
    $$
        I_1(t+h) 
        \le 
        I_2(t+h) 
        + \parens{I_1(t) - I_2(t)}
        + \frac{\sum_{i<h} \parens{\mu_2 + \frac{\mu_1-\mu_2}{2} + \epsilon - 1}}{N_2(t)}
        .
    $$
    For $t \equiv \tau_k$ an exploration episode with $\tau_k$, as $I_1(\tau_k) \le I_2(\tau_k)$ (by definition), we obtain:
    $$
        I_1(\tau_k+h) 
        \le 
        I_2(\tau_k+h) 
        + \frac{\sum_{i<h} \parens{\mu_2 + \frac{\mu_1-\mu_2}{2} + \epsilon - 1}}{N_2(\tau_k)}
        .
    $$
    We see that taking $\epsilon < \frac{\mu_1-\mu_2}2$, the summand is always negative, and as a consequence, $I_1(\tau_k+h) \le I_2(\tau_k+h)$.
    So on $\indicator{\tau_k>T(h)} \cap E_{\tau_k}^{\delta(h)}$, if every pull of the suboptimal arm $a=2$ over $[\tau_k, \tau_k+h)$ provides a full reward $R_t=1$, then $A_{\tau_k+h}=1$. 
    More formally:
    $$
        \indicator{\tau_k \ge T(h)} \cap E_{\tau_k}^{\delta(h)} \cap G_{\tau_k:\tau_k+h} \cap G'_{\tau_k:\tau_k+h}
        =
        \indicator{\tau_k \ge T(h)} \cap E_{\tau_k}^{\delta(h)} \cap G_{\tau_k:\tau_k+h+1} \cap G'_{\tau_k:\tau_k+h}
    $$
    Now choose $\delta := \min_{i \le h} \delta(i)$ and $T := \max_{i \le h} T(i)$. 
    The event $E_{\tau_k} := E_{\tau_k}^\delta \cap \indicator{\tau_k\ge T}$ is $\sigma(H_{\tau_k})$-measurable, and we see that:
    \begin{align*}
        &\Pr(\forall i < h: A_{\tau_k+i} = 2 \mid E_{\tau_k}, U_{\tau_k})
        \\
        &\ge
        \Pr(\forall i < h: A_{\tau_k+i} = 2, R_{\tau_k+i}=1 \mid E_{\tau_k}, U_{\tau_k})
        \\
        & =
        \prod\nolimits_{i<h}
        \Pr \parens{
            R_{\tau_k+i}=1
            \mid A_{\tau_k+i}=2
        }
        \Pr \parens{
            A_{\tau_k+i}=2
            \mid 
            G_{\tau_k:\tau_k+i}, G'_{\tau_k:\tau_k+i},
            E_{\tau_k}, U_{\tau_k}
        }
        \\
        & =
        \prod\nolimits_{i<h}
        \Pr \parens{
            R_{\tau_k+i}=1
            \mid A_{\tau_k+i}=2
        }
        \\
        & = \mu_2^h. 
    \end{align*}
    Moreover, because $\tau_k < \tau_{k+1}$, the event $(\tau_k>T)$ is eventually true as $k \to \infty$, meaning that $\Pr(\liminf_k E_{\tau_k}) = 1$. 
    This establishes the claim. 
\end{proof}

\subsection{Waiting for UCB to Fail: Proof of Proposition~\ref{proposition:ucb monkey}}

\label{section:ucb fails}

We recall the statement below.
\bigskip

\noindent
\textbf{Proposition~\ref{proposition:ucb monkey}}
{
    \it
    Fix $h > 0$ and assume that we are running UCB.
    There exists an increasing sequence of almost-surely finite stopping times $(\sigma_k : k \ge 1)$ s.t.,
    $$
        \Pr(\Reg(\sigma_k; \sigma_k+h) \ge (\mu_1-\mu_2)h) = 1.
    $$
}

\begin{proof}
    Fix $h \ge 0$.
    Let $\ell > \lceil\frac{\mu_1 h}{1-\mu_1}\rceil$ and $\epsilon < \frac{\mu_1-\mu_2}2$.
    Consider $\tau_k$ an exploration episode.
    Assume that $F_{\tau_k:\tau_k+\ell} := (\forall i < \ell: A_{\tau_k+i}=2, R_{\tau_k+i}=1)$ holds, which is of probability at least $\mu_2^\ell$ on the event $E_{\tau_k}$ given by Lemma~\ref{lemma:ucb non agnosticity}. 
    From Lemma~\ref{lemma:ucb key deviations}, almost surely, we have:
    \begin{align*}
        I_1(\tau_k+\ell+i) 
        & \le
        I_2(\tau_k+\ell+i) 
        - \frac{\ell(1-\mu_1)}{N_2(t)}
        + \frac{h\mu_1}{N_2(t)}
        < I_2(\tau_k+\ell+i).
    \end{align*}
    Thus $E_{\tau_k} \cap F_{\tau_k:\tau_k+\ell} \subseteq (\forall i < h: A_{\tau_k+\ell+i} = 2)$ almost surely, and in particular $\Pr(\Reg(\tau_k+\ell;\tau_k+\ell+h) \ge (\mu_1-\mu_2)h \mid E_{\tau_k}, F_{\tau_k:\tau_k+\ell}) = 1$.
    Since $\Pr(F_{\tau_k:\tau_k+\ell} \mid E_{\tau_k}) \ge \mu_2^\ell$ by Lemma~\ref{lemma:ucb non agnosticity}, we deduce by Borel-Cantelli's Lemma that $\Pr(\limsup_k (E_{\tau_k} \cap F_{\tau_k; \tau_k+\ell})) = 1$.
    Hence, define 
    $$
        \sigma_1 := \inf \set{\tau_k+\ell : E_{\tau_k} \cap F_{\tau_k:\tau_k+\ell}},
        \quad
        \sigma_{n+1} := \inf \set{\tau_k+\ell > \sigma_n : E_{\tau_k} \cap F_{\tau_k:\tau_k+\ell}}.
    $$
    We see that $\sigma_n$ is a stopping time.
    Moreover, we have $\Pr(\sigma_n < \infty)$ and $\Pr(\Reg(\sigma_n;\sigma_n+h) \ge (\mu_1-\mu_2)h) = 1$ by construction. 
\end{proof}

\subsection{The Regret of Exploration of UCB: Proof of \cref{theorem:ucb regret of exploration}}

\label{section:ucb exploration}

This section is dedicated to a proof of:
\bigskip

\noindent
\textbf{\cref{theorem:ucb regret of exploration}}
{
    \it
    Let $(X_t:t\ge 1)$ a sequence of i.i.d.~random variables with distribution ${\rm B}(\mu_2)$. 
    Let $\sigma_T$ the stopping time
    $
        T \wedge
        \inf\set{ 
            t \ge 1: 
            -\frac{\mu_1 - \mu_2}2 + \frac 1t \sum_{i=1}^t (X_t - \mu_2) \le 0
        }.
    $
    For all $T\ge 1$, we have $\RegExp({\rm UCB}; T) = \lim_{k \to \infty} \E[\Reg(\tau_k;\tau_k+T)] \ge (\mu_1-\mu_2)\E[\sigma_T]$. 
}

\bigskip
\noindent
    Again, as given by Proposition~\ref{proposition:asymptotic ucb}, the asymptotic regime is denoted 
    $$
        E^\epsilon_t 
        := 
        (\forall a: |\hat \mu_a(t) - \mu_a| < \epsilon)
        \cap 
        \parens{
            N_2(t) = \tfrac{2\cdot(1\pm\epsilon)}{(\mu_1-\mu_2)^2}\log(t)
        }
    $$
    Denote $I_a(t) := \hat \mu_a(t) + \sqrt{2\log(t)/N_a(t)}$ UCB's index of arm $a$, similarly to previous sections. 
    We begin by establishing a variant of Lemma~\ref{lemma:ucb key deviations} for time-periods when only the optimal arm is being pulled. 

\begin{lemma}
    \label{lemma:ucb auxilary key deviations}
    Let $F_{t:t+h} := (\forall i < h, A_{t+i}=1)$.
    Fix arbitrary $\epsilon > 0$ and $h \ge 1$. 
    There exists $\delta, T > 0$ such that, whenever $t > T$ and for all $\ell<h$, on $E_t^\delta \cap F_{t:t+\ell}$ we have:
    $$
        \abs{
            \parens{I_1(t+\ell) - I_2(t+\ell)}
            -
            \parens{
                I_1(t) - I_2(t)
                + \frac{\sum_{i < \ell} (R_{t+i}-\mu_1)}{t}
            }
        }
        \le \frac{2\ell\epsilon}{t}.
    $$
\end{lemma}
\begin{proof}
    The proof is essentially similar to the one of Lemma~\ref{lemma:ucb key deviations}: Approximate $I_a(t+\ell) - I_a(t)$ using equivalents in the asymptotic regime. 
    Using that $N_1(t) \sim t$ and $N_2(t) \sim \frac 2{(\mu_1-\mu_2)^2}\log(t)$, we find that the dominant term in the variations of $I_1(t+\ell) - I_2(t+\ell)$ with respect to $\ell$ is the one coming from the variations of the best arm's empirical estimate. 
    \begin{align*}
        I_1(t+\ell) - I_1(t)
        & = 
        \frac{\sum_{i<\ell}(R_{t+i} - \mu_1)}{t} 
        + \oh\parens{\frac \ell t} 
        \\
        I_2(t+\ell) - I_2(t)
        & =
        \frac{\ell(\mu_1-\mu_2)}{2t\log(t)}
        + \oh\parens{\frac{\ell}{t\log(t)}}.
    \end{align*}
    Quantifying the equivalents with $\epsilon > 0$, we obtain the statement of Lemma~\ref{lemma:ucb auxilary key deviations}. 
\end{proof}

\begin{proof}[Proof of \cref{theorem:ucb regret of exploration}]
    Recall that $(X_t: t \ge 1)$ denotes a sequence of i.i.d.~random variables with distribution ${\rm B}(\mu_2)$.
    Fix $h \ge 1$ and denote the \emph{exploitation episodes} $(\tau'_k)$ as:
    $$
        \tau'_k := \inf\set{t > \tau_k: A_t = 1, A_{t-1}=2}.
    $$
    It is obvious that $\E[\Reg(\tau_k; \tau_k+h)] \le (\mu_1-\mu_2)\E[\min(\tau'_k - \tau_k, h)]$, hence we are ought to bound $\E[\min(\tau'_k - \tau_k, h)]$ which is related to the expected duration of the $k$-th exploration episode clipped to $[0, h]$. 
    From Lemma~\ref{lemma:ucb auxilary key deviations} follows that at the beginning of an exploration episode $\tau_k$ and on $E_{\tau_k}^\epsilon$ for $\epsilon$ (resp.~$\tau_k$) small enough (resp.~large enough), we have:
    $$
        0
        \le 
        I_2(\tau_k) - I_1(\tau_k)
        \le 
        \frac 2{\tau_k}.
    $$
    Furthermore, if $A_{\tau_k+\ell} = 1$, then $I_2(\tau_k+\ell) - I_1(\tau_k+\ell) \le 0$, so combined with Lemma~\ref{lemma:ucb key deviations} and denoting $\nu_0 := \frac{\mu_1+\mu_2}2$, it implies that
    $$
        \sum_{i<\ell} (R_{\tau_k+i} - \nu_0)
        \le 
        2\ell\epsilon + \frac{2N_2(\tau_k)}{\tau_k}
    $$
    where $R_{\tau_k+i} \sim {\rm B}(\mu_2)$.
    Provided that $\tau_k$ is large enough (i.e., that $k$ is large enough), this in particular implies that $\sum_{i<\ell} (R_{\tau_k+i} - \nu_0) \le 3h\epsilon$.
    Since $\epsilon$ can be chosen arbitrarily close to $0$, we deduce that, for all $\epsilon > 0$, 
    $$
        \liminf_{k \to \infty} \E\brackets{\Reg(\tau_k;\tau_k+h)} 
        \ge
        (\mu_1-\mu_2) 
        \E\brackets{
            \inf\set{t \le h: \sum\nolimits_{i<t} (X_t - \nu_0) \le \epsilon}
        }
    $$
    Since $\sum_{i < t} (X_t-\nu_0)$ takes finitely many values when $t \le h$, we have:
    $$
        \inf_{\epsilon > 0}
        \E\brackets{
            \inf\set{t \le h: \sum\nolimits_{i<t} (X_t - \nu_0) \le \epsilon}
        }
        =
        \E\brackets{
            \inf\set{t \le h: \sum\nolimits_{i<t} (X_t - \nu_0) \le 0}
        }
    $$
    This proves the result. 
\end{proof}

\clearpage
\section{General Index Theory}

We write $X_a(t)$ any data relative to the arm $a$, and $X_{-a}(t)$ any data relative to the other arm. 
The index of arm $a$ is thus denoted $I_a$, while $I_{-a}$ denotes the one of the other arm.

\subsection{Proof of Lemma~\ref{lemma:assumptions 1-3}}

\textbf{Lemma~\ref{lemma:assumptions 1-3}}
\textit{
    Assume that $I(-)$ satisfies \ASSUMPTION{1-3}.
    Then, for $a =1,2$, $\hat \mu_a(t) \to \mu_a$ a.s. 
}

\begin{proof}[Proof of Lemma~\ref{lemma:assumptions 1-3}]
\STEP{1}
    We start by showing that both arms are visited infinitely often, that is, for $a = 1, 2$ and for a fixed arbitrary $n$, $\Pr(\exists t: N_a(t) \ge n) = 1$. 
    By the Strong Law of Large Number (SLLN, or just time-uniform concentration inequalities), the result will follow. 
    Consider the complementary event.
    Remark that if $N_a(t) < n$, there must be $s \in \set{t-n, \ldots, t}$ such that $A_s \ne a$. 
    So, we have, for $\delta > 0$ small enough, 
    \begin{align*}
        & \Pr \parens{\forall t: N_a(t) < n}
        \\
        & \le
        \Pr \parens{\forall t, \exists s \ge t - n: N_a(s) < n, A_s \ne a}
        \\
        & = \lim_{t \to \infty}
        \Pr \parens{
            \exists s \ge t-n:
            N_a(s) < n,
            I_a(\hat \mu(s), N_a(s), s) \le I_{-a}(\hat \mu(s), N_{-a}(s), s)
        }
        \\
        & = \lim_{t \to \infty}
        \Pr \parens{
            \exists s \ge t-n:
            I_a(\hat \mu(s), n, s) \le I_{-a}(\hat \mu(s), s-n, s)
            , N_{-a}(s) \ge s - n
        }
        & \ASSUMPTION{1}
        \\
        & \le \lim_{t \to \infty}
        \Pr \parens{
            \exists s \ge t - n:
            I_a(\hat \mu(s), n, s) \le I_{-a}(\hat \mu(s), s-n, s),
            \abs{\hat \mu_{-a}(s) - \mu_{-a}} \le \delta
        }
        & \text{(SLLN)}
        \\
        & \le \lim_{t \to \infty}
        \Pr \parens{
            \exists s \ge t - n:
            I_a(0, \mu_{-a}-\delta, n, s) \le I_{-a}(\mu_{-a}+\delta, 1, s-n, s)
        }
        & \ASSUMPTION{1}
        \\
        & \le \lim_{t \to \infty}
        \Pr \parens{
            \exists s \ge t - n:
            I_a(0, \mu_{-a}-\delta, n, s) \le I(\mu_1, 1) + \oh_{\delta \to 0}(1)
        }
        & \ASSUMPTION{3}
        \\
        & = 0.
        & \ASSUMPTION{2}
    \end{align*}

    \STEP{2}
    Since all arms are pulled infinitely often, the empirical estimates must converge to the mean values by the Strong Law of Large Numbers. 
\end{proof}

\subsection{Proof of Lemma~\ref{lemma:index asymptotic regime}}

\textbf{Lemma~\ref{lemma:index asymptotic regime}}
\textit{
    If $I(-)$ satisfies \ASSUMPTION{1-6}, then $(\hat \mu_1(t), \hat \mu_2(t), N_1(t), N_2(t)) \sim (\mu_1, \mu_2, t, n_2(t))$ a.s. 
    The sequence $t \mapsto (\mu_1, \mu_2, t, n_2(t))$ will be called the \emph{asymptotic regime}.
}

\begin{proof}[Proof of Lemma~\ref{lemma:index asymptotic regime}]
    Since $n_2(t)$ is sublinear, we only have to show the property on $N_2(t)$ and everything will follow. 

    \STEP{1}
    Let $\epsilon > 0$ and focus on $a = 2$ the suboptimal arm. 
    For conciseness, denote $c := 1 - \epsilon$.
    Similarly to the previous point, remark that if $N_2(s) < cn_2(s)$, there must be some $u \in \set{s - cn_2(s), \ldots, s}$ when $A_u \ne 2$. 
    Let $F_t^\delta := (\forall a, \abs{\hat \mu_a(t) - \mu_a} < \delta)$ the concentration event, proved to hold eventually, as given by Lemma~\ref{lemma:assumptions 1-3}. 
    We then have:
    \begin{align*}
        & \Pr \parens{\forall t, \exists s \ge t: N_2(s) < c n_2(s)}
        \\
        & \le
        \lim_{t \to \infty}
        \Pr \parens{
            \exists s \ge t, \exists u \ge s - c n_2(s):
            N_2(s) \le c n_2(s), 
            I_2(u) \le I_1(u)
        }
        \\
        & \le \lim_{t \to \infty}
        \Pr \parens{
            \exists s \ge t, \exists u \ge s - c n_2(s):
            N_2(s) \le c n_2(s), 
            I_2(u) \le I_1(u),
            F_s^\delta
        }
        & (\text{Lem.~\ref{lemma:assumptions 1-3}})
        \\
        & \le 
        \lim_{t \to \infty}
        \Pr \parens{
            \exists s \ge t, \exists u \ge s - c n_2(s):
            \begin{array}{c} I(\mu_2-\delta, \mu_1+\delta, cn_2(s), u) \\ \le \\ I(\mu_1+\delta, \mu_2-\delta, u-cn_2(s), u) \end{array}
        }
        & \ASSUMPTION{1}
        \\
        & \le 
        \lim_{t \to \infty}
        \Pr \parens{
            \exists s \ge t, \exists u \ge s - c n_2(s):
            \begin{array}{c}
                I(\mu_2-\delta, \mu_1+\delta, cn_2(s), u)
                \\ \le \parens{1 + \underset{\delta \to 0}\oh(1) + \underset{t \to \infty}\oh(1)} I(\mu_1, \mu_2)
            \end{array}
        }
        & \ASSUMPTION{3,6}
        \\
        & \le 
        \lim_{t \to \infty}
        \Pr \parens{
            \exists s \ge t, \exists u \ge s - c n_2(s):
            \begin{array}{c}
                I(\mu_2-\delta, \mu_1+\delta, (1-\epsilon)n_2(s), u)
                \\ \le \parens{1+\oh(1)} 
                I(\mu_2-\delta, \mu_1+\delta, n_2(s), s)
            \end{array}
        }
        & \ASSUMPTION{4}
        \\
        & \le \lim_{t \to \infty}
        \Pr \parens{
            \exists s \ge t, \exists u \ge s - c n_2(s):
            \begin{array}{c}
                (1+\ell(\epsilon))
                I(\mu_2-\delta, \mu_1+\delta, n_2(s), s)
                \\ \le \parens{1 + \oh(1)} 
                I(\mu_2-\delta, \mu_1+\delta, n_2(s), s)
            \end{array}
        }
        & \ASSUMPTION{5}
        \\
        & \le \lim_{t \to \infty}
        \Pr \parens{
            \exists s \ge t, \exists u \ge s - c n_2(s):
            1 + \ell(\epsilon)
            \le 1 + \underset{\delta \to 0}\oh(1) + \underset{t \to \infty}\oh(1)
        }
        \\
        & = 0.
        & (\delta \to 0)
    \end{align*}

    \STEP{2}
    Let $\epsilon > 0$ and focus on $a = 2$ the suboptimal arm.
    This time, denote $c := 1 + \epsilon$. 
    The analysis is mostly similar, but the initial decomposition starts differently.
    Remark that if $N_2(s) > c n_2(s)$, then there must be $u \in \set{c n_2(s), \ldots, s}$ such that $A_u = 2$. 
    So, 

    \begin{align*}
        & \Pr (\forall t, \exists s \ge t: N_2(s) \ge c n_2(s))
        \\
        & \le \lim_{t \to \infty}
        \Pr \parens{
            \exists s \ge t, \exists u \in [cn_2(s), s]:
            N_2(u) \ge c n_2(s), 
            A_u = 2
        }
        \\
        & \le \lim_{t \to \infty}
        \Pr \parens{
            \exists s \ge t, \exists u \in [cn_2(s), s]:
            N_2(u) \ge c n_2(s), 
            I_2(u) \ge I_1(u)
        }
        \\
        & \le \lim_{t \to \infty}
        \Pr \parens{
            \exists s \ge t, \exists u \in [cn_2(s), s]:
            I_2(\hat \mu(u), c n_2(u), u) \ge I_1(\hat \mu(u), u, u)
        }
        & \ASSUMPTION{1}
        \\
        & \le \lim_{t \to \infty}
        \Pr \parens{
            \exists s \ge t, \exists u \in [cn_2(s), s]:
            I_2(\hat \mu(u), c n_2(u), u) \ge I_1(\hat \mu(u), u, u),
            F_u^\delta
        }
        \\
        & \le \lim_{t \to \infty}
        \Pr \parens{
            \exists s \ge t, \exists u \in [cn_2(s), s]:
            I_2(\hat \mu(u), c n_2(u), u) \ge (1 + \oh(1)) I(\mu_1, \mu_2),
            F_u^\delta
        }
        & \ASSUMPTION{3}
        \\
        & \le \lim_{t \to \infty}
        \Pr \parens{
            \exists s \ge t, \exists u \in [cn_2(s), s]:
            F_u^\delta,
            \begin{array}{c}
                I_2(\hat \mu(u), (1+\epsilon) n_2(u), u) \\
                \ge \\ (1 + \oh(1)) I_2(\hat \mu(u), n_2(u), u)
            \end{array}
        }
        & \ASSUMPTION{4}
        \\
        & \le \lim_{t \to \infty}
        \Pr \parens{
            \exists s \ge t, \exists u \in [cn_2(s), s]:
            1 - \ell(\epsilon)
            \ge 1 + \underset{\delta \to 0}o(1) + \underset{t\to \infty}\oh(1)
        }
        & \ASSUMPTION{5}
        \\
        & = 0.
        & (\delta \to 0)
    \end{align*}
    So, $t\mapsto N_2(t)$ converges to $t \mapsto n_2(t)$ almost-surely for the asymptotic topology.
\end{proof}

\bibliography{bibliography}

\end{document}